\documentclass{article}
\usepackage{aux/macros}
\usepackage{enumitem}

\title{Emergence in non-neural models: \\ grokking modular arithmetic via average gradient outer product}

\author{Neil Mallinar$^{1,2}$
  \quad 
  Daniel Beaglehole$^1$
  \quad
  Libin Zhu$^1$\\
  Adityanarayanan Radhakrishnan$^2$
  \quad 
  Parthe Pandit$^3$
  \quad 
  Mikhail Belkin$^1$\\~\\
  \small $^1$UC San Diego \quad $^2$The Broad Institute of MIT and Harvard \quad $^3$IIT Bombay\\~\\
  \small \texttt{
  \{nmallina,dbeaglehole,libinzhu,mbelkin\}@ucsd.edu}\\
  \small \texttt{aradha@mit.edu}~~;~~~
  \small \texttt{pandit@iitb.ac.in}
}
\date{}

\begin{document}

\maketitle

\begin{abstract}
Neural networks trained to solve modular arithmetic tasks exhibit \textit{grokking}, a phenomenon where the test accuracy starts improving  long after the model achieves $100\%$ training accuracy in the training process.  It is often taken as an example of ``emergence'', where model ability manifests sharply  through a phase transition. In this work, we show that the phenomenon of grokking is not specific to neural networks nor to gradient descent-based optimization. Specifically, we show that this phenomenon occurs when learning modular arithmetic with Recursive Feature Machines (RFM), an iterative algorithm that uses the Average Gradient Outer Product (AGOP) to enable task-specific feature learning with general machine learning models.  When used in conjunction with kernel machines, iterating RFM results in a fast transition from random, near zero, test accuracy to perfect test accuracy. This transition cannot be predicted 
from the training loss, which is identically zero, nor 
from the test loss, which remains constant in initial iterations.  Instead, as we show, the transition is completely determined by feature learning: RFM gradually learns block-circulant features to solve modular arithmetic.  
Paralleling the results for RFM, we show that neural networks that solve modular arithmetic also learn block-circulant features. Furthermore, we present theoretical evidence that RFM uses such block-circulant features to implement the \textit{Fourier Multiplication Algorithm}, which prior work posited as the generalizing solution neural networks learn on these tasks.  Our results demonstrate that emergence can result purely from learning task-relevant features and is not specific to neural architectures nor gradient descent-based optimization methods.  Furthermore, our work provides more evidence for AGOP as a key mechanism for feature learning in neural networks.\footnote{Code is available at \url{https://github.com/nmallinar/rfm-grokking}.}
\end{abstract}

\section{Introduction}

In recent years the  idea of ``emergence'' has become an important narrative in machine learning. While there is no broad agreement on the definition~\cite{PositionPaper},  it is often argued that ``skills''  emerge during the training process once certain data size, compute, or model size thresholds are achieved~\citep{wei2022emergent, AroraEmergent}. Furthermore, these skills are believed to appear rapidly, exhibiting sharp and seemingly unpredictable improvements in performance at these thresholds.  One of the simplest and most striking examples supporting this idea is ``grokking'' modular arithmetic~\citep{power2022grokking,nanda2023progress}. 
A neural network trained to predict modular addition or another arithmetic operation on a fixed data set rapidly transitions from near-zero    to perfect ($100\%$)   test accuracy at a certain point in the optimization process. Surprisingly, this transition point occurs long after perfect {\it training accuracy} is achieved.
Not only is this contradictory to the traditional wisdom regarding overfitting but, as we will see below, some aspects of grokking do not fit neatly with our modern understanding of ``benign overfitting''~\cite{bartlett2021deep, belkin2021fit}.

Despite a large amount of recent work on emergence and, specifically, grokking, (see, e.g., \citep{power2022grokking, liu2023omnigrokgrokkingalgorithmicdata, nanda2023progress, thilak2022slingshotmechanismempiricalstudy, furuta2024interpreting, miller2024grokkingneuralnetworksempirical}), the  nature or even existence of the emergent phenomena remains contested.
For example, the recent paper~\cite{schaeffer2023are} suggests that the  rapid emergence of skills may be a ``mirage'' due to the mismatch between the  discontinuous metrics used for evaluation, such as accuracy, and the continuous loss used in training.  The authors argue that, in contrast to accuracy, the test (or validation) loss or some other suitably chosen metric may decrease gradually throughout training and thus provide a useful measure of progress.  
Another possible progress measure is the training loss. As  SGD-type optimization algorithms generally result in a gradual decrease of the training loss, one may posit that skills appear once the training loss falls below a certain threshold in the optimization process. Indeed, such a conjecture is in the spirit of  classical generalization theory, which considers the training loss to be a useful proxy for the test performance~\cite{mohri2018foundations}. 

In this work, we show that sharp emergence in modular arithmetic  arises entirely from feature learning, independently of  other aspects of modeling and training, and is not predicted by the standard measures of progress.  We then clarify the nature of feature learning leading to the emergence of  skills in modular arithmetic. We discuss these contributions in further detail below.

\paragraph{Summary of the contributions.}
We demonstrate empirically that grokking modular arithmetic
\begin{itemize}
    \item is not specific to neural networks;
    \item is not tied to gradient-based optimization methods;
    \item is not  predicted by training or test loss\footnote{We note that for neural networks trained by SGD, emergence cannot be decoupled from training loss, as non-zero loss is required for training to occur at all.}, let alone accuracy.\footnote{Note that zero test/train  loss implies perfect test/train accuracy.}
\end{itemize}

\begin{figure}[t!]
    \centering
    \includegraphics[width=0.9\textwidth]{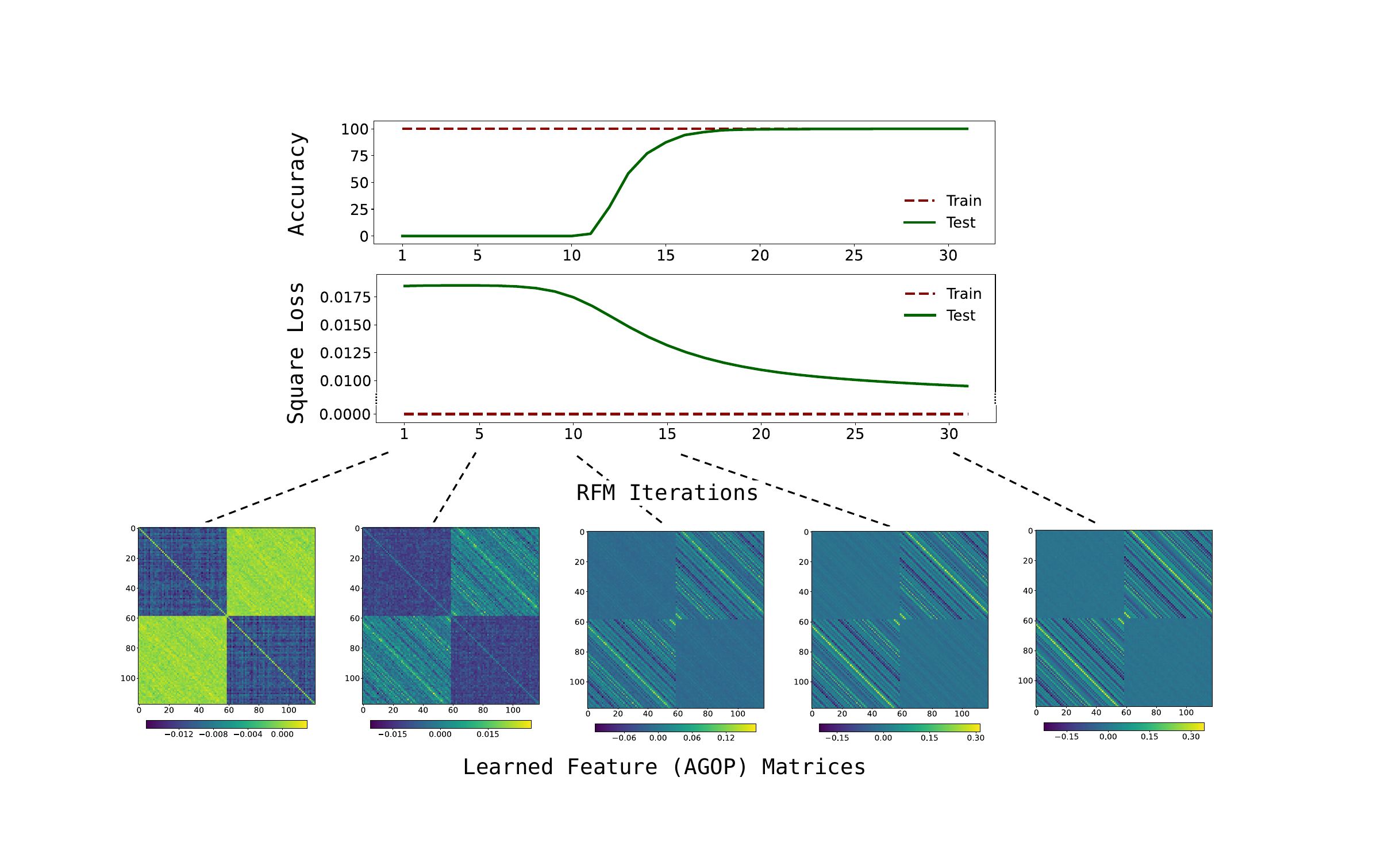}
    \caption{Recursive Feature Machines grok the modular arithmetic task ${f^*(x,y) = (x + y) \mod 59}$.}
    \label{fig:intro_fig_p59}
\end{figure}

Specifically, we show grokking for Recursive Feature Machines (RFM)~\cite{rfm_science}, an algorithm that iteratively uses the Average Gradient Outer Product (AGOP) to enable task-specific feature learning in general machine learning models.  
In this work, we use RFM to enable feature learning in kernel machines, which are a class of predictors with no native mechanism for feature learning.  In this setting, RFM iterates between three steps: (1) training a kernel machine, $f$, to fit training data; (2) computing the AGOP matrix of $f$,  $M$, over the training data to extract task-relevant features; and (3) transforming input data, $x$, using the learned features via the map $x \to M^{s/2}x$ for a matrix power $s > 0$ (see Section~\ref{sec: preliminaries} for details).

In Fig.~\ref{fig:intro_fig_p59} we give a representative example of   RFM grokking  modular addition, despite not using any gradient-based optimization methods and achieving perfect (numerically zero) training loss at every iteration. 
We see that during the first few iterations both the test loss and and test accuracy remain at the constant (random) level. However, around iteration  $10$, the test loss starts improving, and a few iterations later, test accuracy quickly transitions to $100\%$. We also observe that even early in the iteration, structure emerges in AGOP feature matrices (see Fig.~\ref{fig:intro_fig_p59}).  The gradual appearance of structure in these feature matrices is particularly striking given that the training loss is identically zero at every iteration and that the test loss does not significantly change until at least iteration eight. 
The striped patterns observed in feature matrices correspond to matrices whose sub-blocks are circulant with entries that are constant along the diagonals.\footnote{More precisely the entries of circulants are constant along the ``long'' diagonals which wrap around the matrix.  Feature matrices may also be block Hankel matrices which are constant on anti-diagonals. Unless the distinction is important, we will generally refer to all such matrices as circulant.} Such {\it circulant feature matrices} are key to learning modular arithmetic. In Section~\ref{sec: rfm emergence} we demonstrate that standard kernel machines using {\it random} {circulant features} easily learn modular operations. {As these random circulant matrices are generic, we argue that no additional structure is required to solve modular arithmetic.}

To demonstrate that the feature matrices evolve toward this structure (including for multiplication and division under an appropriate re-ordering of the input coordinates),
we introduce two  ``hidden progress measures'' (cf.~\cite{barak2022hidden}):
\begin{enumerate}
\item {\it Circulant deviation}, which measures how far the diagonals of a given matrix are from being constant. \
\item {\it AGOP alignment}, which measures similarity between the feature matrix at iteration $t$  and the AGOP of a fully trained model.
\end{enumerate}
As we demonstrate, both of these measures show gradual (initially nearly linear) progress toward a model that generalizes. 

We further argue that emergence in fully connected neural networks trained on modular arithmetic identified in prior work \citep{gromov2023grokking, mlpGrokking22} is analogous to that for RFM and is exhibited through the AGOP (see Section~\ref{sec: neural nets}). 
Indeed, by visualizing covariances of network weights, we observe that these models also learn block-circulant features to grok modular arithmetic. 
We demonstrate that these features are highly correlated with the AGOP of neural networks, corroborating prior observations from~\cite{rfm_science}.
Furthermore, paralleling our observations for RFM, our two proposed progress measures indicate gradual  progress toward a generalizing solution during neural network training.  Finally, just as for RFM, we demonstrate that training neural networks on data transformed by  random block-circulant matrices dramatically decreases training time needed to learn modular arithmetic. 

Why are these learned block circulant features effective for modular arithmetic? 
We provide supporting theoretical evidence that such circulant features result in kernel machines implementing the Fourier Multiplication Algorithm (FMA) for modular arithmetic (see Section~\ref{sec: fourier multiplication}).  For the case of neural networks, several prior works have argued empirically and theoretically that neural networks learn to implement the FMA to solve modular arithmetic~\citep{nanda2023progress, varma2023explaininggrokkingcircuitefficiency, morwani2024featureemergencemarginmaximization}.  Thus, while kernel-RFM and neural networks utilize different classes of predictive models, our results suggest that these models discover similar algorithms for implementing modular arithmetic. In particular, these results imply that the two methods have the same out-of-domain generalization properties when the entries of input vectors contain real numbers instead of $0$'s or $1$'s in the training data. 

Overall, by decoupling feature learning from predictor training, our results provide evidence for the emergent properties of machine learning models arising purely as a consequence of their ability to learn features.  We hope that our work will help isolate the underlying mechanisms of emergence and shed light on the key practical concern of how, when, and why these seemingly unpredictable transitions occur.

\paragraph{Paper outline.} Section~\ref{sec: preliminaries} contains the necessary concepts and definitions. In Section~\ref{sec: rfm emergence}, we demonstrate emergence with RFM and show AGOP features consist of circulant blocks. In Section~\ref{sec: neural nets}, we show that the neural network features are circulant and are captured by the AGOP. In Section~\ref{sec: fourier multiplication}, we prove that kernel machines learn the FMA with circulant features. 
We provide a broad discussion and conclude in Section~\ref{sec: discussion}.

\section{Preliminaries}
\label{sec: preliminaries}
\paragraph{Learning modular arithmetic.} Let $\Zp = \Z/p\Z$ denote the field of integers modulo a prime $p$ and let $\Zp^* = \Zp \backslash \{0\}$.  We learn modular functions $f^*(a, b) = g(a,b) \mod p$  where $f^* : \Zp \times \Zp \rightarrow \Zp$, $a, b\in \Zp$, and $g : \Z \times \Z \rightarrow \Z$ is an arithmetic operation on $a$ and $b$, e.g. $g(a, b) = a + b$. Note that there  are  $p^2$ discrete input pairs ($a, b$)  for all modular operations except for $f^*(a,b) = (a \div b) \mod p$, which has $p(p-1)$ inputs as the denominator cannot be $0$. 

To train models on modular arithmetic tasks, we construct input-label pairs by one-hot encoding the input and label integers. Specifically, for every pair $a, b \in \Zp$, we write the input as $\e_a \oplus \e_b \in \Real^{2p}$ and the output as $\e_{f^*(a,b)} \in \R^p$, where $\e_i \in \Real^p$ is the $i$-th standard basis vector in $p$ dimensions and $\oplus$ denotes concatenation. For example, for addition modulo $p=3$,  the equality $1+2=0 \mod 3 $ is encoded as an input/label pair of vectors in $\R^6$ and $\R^3$, respectively:
$$
\text{Input: } (\underbrace{0 ~ 1 ~ 0}_{a=1}~ \underbrace{ 0 ~ 0 ~ 1}_{b=2}) ~ ; ~ \text{Label:}\underbrace{(1 ~ 0 ~0)}_{a+b\mod3=0} .
$$
The training dataset consists of a random subset of $n = r \times N$ input/label pairs, where $r$ is termed the \emph{training fraction} and $N = p^2$ or $p(p-1)$ is the number of possible discrete inputs.

\paragraph{Complex inner product and Discrete Fourier Transform (DFT).}  In our theoretical analysis in Section~\ref{sec: fourier multiplication}, we will utilize the following notions of complex inner product and DFT.  The complex inner product $\langle \cdot, \cdot \rangle_{\mathbb{C}}$ is a map from $\mathbb{C}^{d} \times \mathbb{C}^{d} \to \mathbb{C}$ of the form 
\begin{align}
    \label{eq: Complex inner product}
    \langle u, v \rangle_{\mathbb{C}} = u\tran \bar{v}~,
\end{align}
where $\bar{v}_j$ is the complex conjugate of $v_j$.  Let $i = \sqrt{-1}$ and let $\omega = \exp(\frac{-2\pi i}{d})$.  The DFT is the map $\mathcal{F}: \mathbb{C}^{d} \to \mathbb{C}^d$ of the form $\mathcal{F}(u) = F u,$ where $F \in \mathbb{C}^{d \times d}$ is a unitary matrix with $F_{ij} = \frac{1}{\sqrt{d}}\omega^{ij}$. In matrix form, $F$ is given as 
\begin{align}
\label{eq: DFT}
F = \frac{1}{\sqrt{d}}
\begin{pmatrix}
1 & 1 & 1 & \cdots & 1 \\
1 & \omega & \omega^2 & \cdots & \omega^{d-1} \\
1 & \omega^2 & \omega^4 & \cdots & \omega^{2(d-1)} \\
\vdots & \vdots & \vdots & \ddots & \vdots \\
1 & \omega^{d-1} & \omega^{2(d-1)} & \cdots & \omega^{(d-1)(d-1)}
\end{pmatrix}
~.
\end{align}

\paragraph{Circulant matrices.} As we will show, the features that RFMs and neural networks learn in order to solve modular arithmetic contain blocks of \emph{circulant matrices}, which are defined as follows.  Let $\sigma : \Real^p \rightarrow \Real^p$ be the cyclic permutation  which acts on a vector $u\in\R^p$ by shifting its coordinates by one cell to the right:
\begin{align}
    [\sigma(u)]_j = u_{j-1 \mod p}~,
\end{align}    
for $j \in [p]$. We write the $\ell$-fold composition of this map $\sigma^\ell(u) \in \Real^p$ with entries $[\sigma^\ell(u)]_{j} = u_{j-\ell \mod p}$. A circulant matrix $C \in \Real^{p \times p}$ is determined by a vector $\c = [c_0, \ldots, c_{p-1}] \in \Real^p$, and has form:
\[
C = \begin{pmatrix}
    \horzbar & \c & \horzbar\\
    \horzbar & \sigma(\c) & \horzbar\\
    \vdots & \vdots & \vdots \\
    \horzbar & \sigma^{p-1}(\c) & \horzbar
\end{pmatrix}~.
\]

Namely, circulant matrices are those in which each row is shifted one element to the right of its preceding row.  Hankel matrices are also determined by a vector $\c$, but the rows are $\c, \sigma^{-1}(\c), \ldots, \sigma^{-(p-1)}(\c)$.  While circulant matrices have constant diagonals, Hankel matrices have constant anti-diagonals. To ease terminology, we will use the word circulant to refer to Hankel or circulant matrices as previously defined. 

\begin{algorithm}[t!]
\caption{Recursive Feature Machine (RFM) \citep{rfm_science}}\label{alg: RFM}
\begin{algorithmic}
\Require $X, y, k, T, L$ \Comment{Train data: $(X, y)$, base kernel: $k$, iters.: $T$, matrix power: $s$, and bandwidth: $L$}
\State $M_0 = I_{d}$
\For{$t = 0,\ldots,T-1$}
    \State Solve $\alpha \leftarrow k(X,X;M_t)^{-1}y$ \Comment{$f^{(t)}(x) = k(x,X;M_t)\alpha$}
    \State $M_{t+1} \leftarrow [\AGOPk(f^{(t)})]^{s}$ 
\EndFor
\\\Return $\alpha, M_{T-1}$ \Comment{Solution to kernel regression: $\alpha$, and feature matrix: $M_{T-1}$}
\end{algorithmic}
\end{algorithm}

\paragraph{Average Gradient Outer Product (AGOP).}
The AGOP matrix, which will be central to our discussion, is defined as follows.   

\begin{definition}[AGOP] 
Given a predictor $f: \R^d \to \R^c$ with $c$ outputs, $f(x) \equiv [f_0(x), \ldots, f_{c-1}(x)]$, let $\dd{f(x')}{x} \in \R^{d \times c}$ be the Jacobian of $f$ evaluated at some point $x' \in \R^d$ with entries $[\dd{f(x')}{x}]_{\ell,s} = \dd{f_\ell(x')}{x_s}$.\footnote{While the Jacobian is typically defined as a linear map from $\mathbb{R}^{d} \to \mathbb{R}^{c}$, we will use its transpose in this work. When $c=1$, this is simply the gradient of $f$.} Then, for $f$ trained on a set of data points $\{x^{(j)} \}_{j=1}^n$, with $x^{(j)} \in \R^d$, the Average Gradient Outer Product (AGOP), $\AGOP$, is defined as,
\begin{align}
    \AGOP(f; \{x^{(j)}\}_{j=1}^{n}) &= \frac{1}{n} \sum_{j=1}^n \dd{f(x^{(j)})}{x} \dd{f(x^{(j)})}{x}\tran \in \R^{d \times d}.
\end{align}
\end{definition}
For simplicity, we omit the dependence on the dataset in the notation. 
Top eigenvectors of AGOP can be viewed as the ``most relevant'' input features,  those input directions that influence the output of a general predictor (for example, a kernel machines or a neural network) the most. As a consequence, the AGOP can be viewed as a task-specific transformation that can be used to amplify relevant features and improve sample efficiency of machine learning models.

Indeed, a line of prior works~\citep{HsuEGOP, TrivediEGOP, SpokoinyEGOP, KaneEGOP} have used the AGOP to improve the sample efficiency of predictors trained on multi-index models, a class of predictive tasks in which the target function depends on a low-rank subspace of the data. 
The AGOP was recently shown to be a key mechanism for understanding feature learning in neural networks and enabling feature learning in general machine learning models~\cite{rfm_science}.
Though the study of AGOP has been motivated by these multi-index examples, we will see that the AGOP can be used to recover useful features for modular arithmetic that are, in fact, not low-rank.

\paragraph{AGOP and feature learning in neural networks.} Prior work~\cite{rfm_science} posited that AGOP was a mechanism through which neural networks learn features.  
In particular, the authors posited the \textit{Neural Feature Ansatz (NFA)} stating that for any layer $\ell$ of a trained neural network with weights $W_{\ell}$, the \textit{Neural Feature Matrix (NFM)}, $W_{\ell}^T W_{\ell}$, are highly correlated to the AGOP of the model computed with respect to the input of layer $\ell$.  The NFA suggests that neural networks learn features at each layer by utilizing the AGOP.  For more details on the NFA, see Appendix~\ref{app: nfa}.

\begin{figure}[t!]
    \centering
    \includegraphics[width=1.0\textwidth]{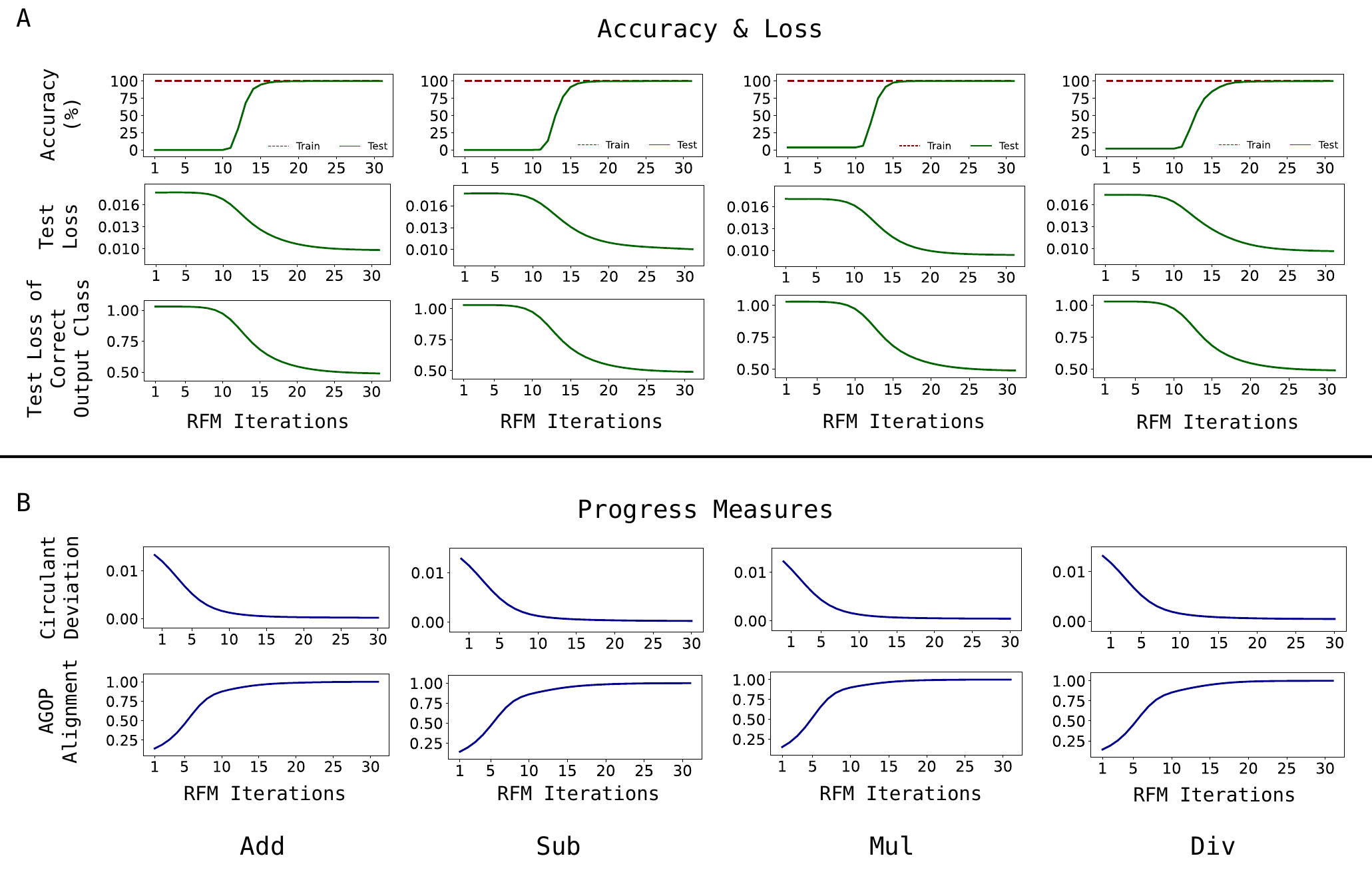}
    \caption{RFM with the quadratic kernel on modular arithmetic with modulus $p = 61$ trained for 30 iterations.
    (A) Test accuracy, test loss (mean squared error) over all output coordinates, and test loss of the correct class output coordinate do not change in the first 8 iterations and then, sharply transition.
    (B) Circulant deviation and AGOP alignment show gradual progress towards generalizing solutions despite accuracy and loss metrics not changing in the initial iterations. For multiplication (Mul) and division (Div), circulant deviation is measured with respect to the feature sub-matrices after reordering by the discrete logarithm.}
    \label{fig:p61_kernel_curves}
\end{figure}

\paragraph{Recursive Feature Machine (RFM).} Importantly, AGOP can be computed for any differentiable predictor, including those such as kernel machines that have no native feature learning mechanism. As such, the authors of~\cite{rfm_science} developed an algorithm known as RFM, which iteratively uses the AGOP to extract features. Below, we present the RFM algorithm used in conjunction with kernel machines.  Suppose we are given data samples $(X, y) \in \mathbb{R}^{n \times d} \times \mathbb{R}^{n}$ where $X$ contains $n$ samples denoted $\{x^{(j)} \}_{j=1}^{n}$. Given an initial symmetric positive-definite matrix $M_0 \in \mathbb{R}^{d \times d}$, and Mahalanobis kernel $k(\cdot, \cdot\,;M): \mathbb{R}^{d} \times \mathbb{R}^{d} \to \mathbb{R}$, RFM iterates the following steps for $t \in [T]$:
\begin{align}
    &\textit{Step 1 (Predictor training): } f^{(t)}(x) = k(x, X ; M_t)\alpha  ~\text{with}~ \alpha = k(X, X; M_t)^{-1} y ~;  \\
    &\textit{Step 2 (AGOP update): } M_{t+1} = [G(f^{(t)})]^{s}~;
\end{align}
where $s > 0 $ is a matrix power and  $k(X,X;M) \in \Real^{n \times n}$ denotes the matrix with entries $[k(X,X;M)]_{j_1j_2} = k(x^{(j_1)}, x^{(j_2)}; M)$ for $j_1, j_2 \in [n]$.  In this work, we select $s = \frac{1}{2}$ for all experiments (see Algorithm~\ref{alg: RFM}).  We use the following two Mahalanobis kernels: (1) the quadratic kernel, $k(x, x'; M) = \round{x\tran M x'}^2$ ; and (2) the Gaussian kernel $k(x, x'; M) = \exp\round{-\frac{\snorm{x - x'}_M^2}{L}}$, where for $z \in \mathbb{R}^{d}$,  $\snorm{z}_M^2 = z^\top M z$, and $L$ is the bandwidth.

\begin{wrapfigure}{R}{0.5\textwidth}
    \centering
    \includegraphics[width=0.5\textwidth]{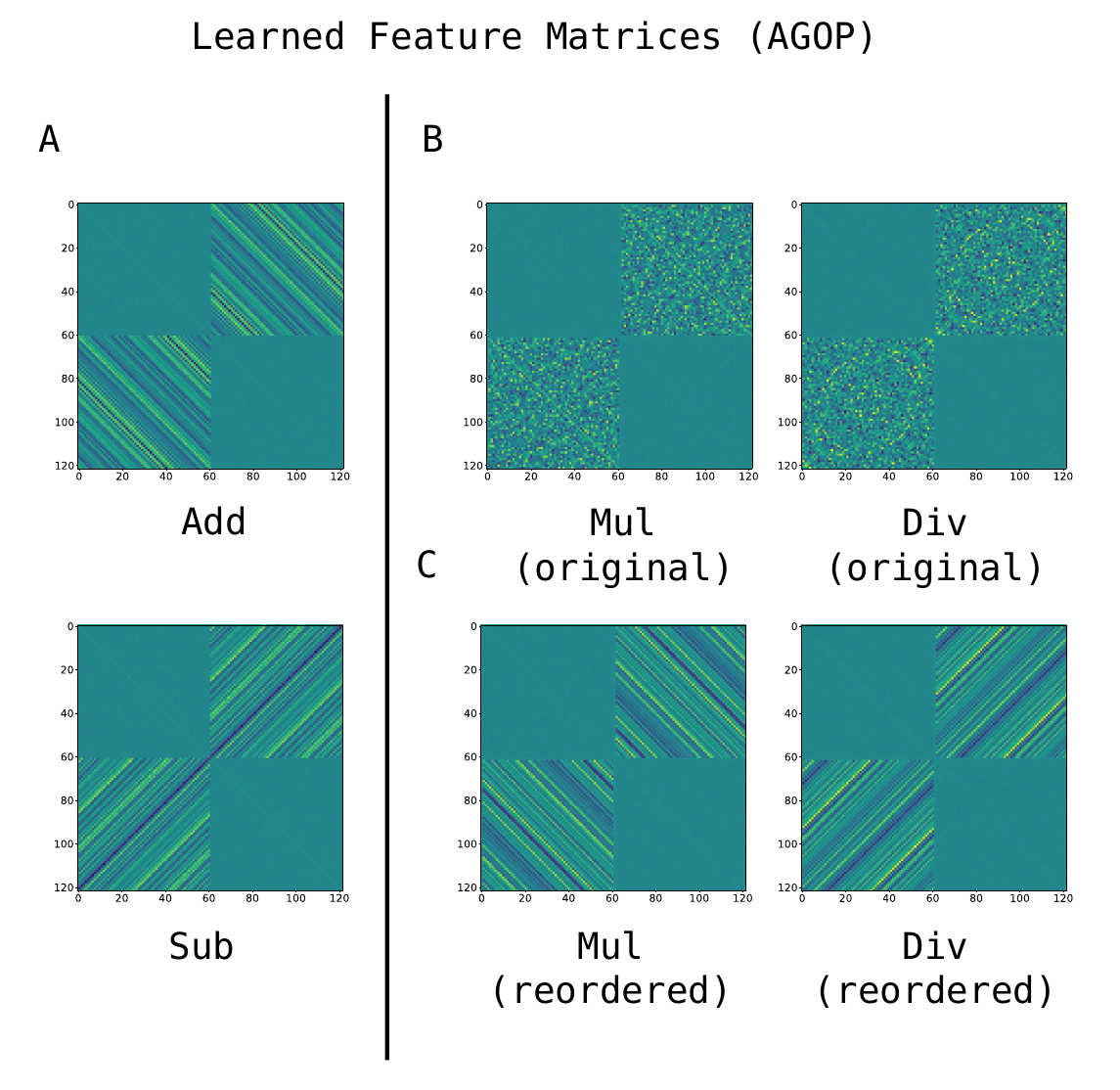}
\caption{RFM with the quadratic kernel for modular arithmetic tasks with modulus $p = 61$.
    (A) The square root of the kernel AGOPs for addition (Add), subtraction (Sub) visualized without their diagonals to emphasize circulant structure in off-diagonal blocks.
    (B)  Square root of the kernel AGOPs for multiplication (Mul), division (Div).
    (C) For Mul and Div, rows and columns of the individual sub-matrices of the AGOP are re-ordered by the discrete logarithm base $2$, revealing the block-circulant structure in the features.}
    \label{fig:p61_kernel_agops}
    \vspace{-3em}
\end{wrapfigure}

\section{Emergence with Recursive Feature Machines}
\label{sec: rfm emergence}

In this section, we will demonstrate that RFM (Algorithm~\ref{alg: RFM})  exhibits sharp transitions in performance on four modular arithmetic tasks (addition, subtraction, multiplication, and division) due to the emergence of features, which are block-circulant matrices.  

We will use a modulus of $p = 61$ and train RFM  with quadratic and Gaussian kernel machines (experimental details are provided in Appendix~\ref{apdx:model_training_details}).  As we solve kernel ridgeless regression exactly,  all iterations of RFM result in zero training loss and 100\% training accuracy.  The top two rows of Fig.~\ref{fig:p61_kernel_curves}A  show that the first several iterations of RFM result in  near-zero test accuracy and approximately constant, large test loss.  Despite these standard progress measures initially not changing, continuing to iterate RFM leads to a dramatic, sharp increase to 100\% test accuracy and a corresponding decrease in the test loss later in the iteration process.

\paragraph{Sharp transition in loss of correct output coordinate.} 
It is important to note that our total loss function is the square loss averaged over $p=61$ classes.  It is thus plausible that, due to averaging, the near-constancy of the total square loss over the first few iterations conceals steady improvements in the predictions of the correct class. However, it is not the case.  In Fig.~\ref{fig:p61_kernel_curves}A (third row) we show that the test loss for the output coordinate (logit) corresponding to the correct class closely tracks the total test loss. 

\paragraph{Emergence of block-circulant features in RFM.}  In order to understand how RFM generalizes, we visualize the $2p \times 2p$ matrix given by the square root of the AGOP of the final iteration of RFM.  We refer to these matrices as \textit{feature matrices}.  We first visualize the feature matrices for RFM trained on modular addition/subtraction in Fig.~\ref{fig:p61_kernel_agops}A. Their visually-evident striped structure suggests the following more precise characterization.

\begin{observation}[Block-circulant features]
Feature matrix $M^* \in \Real^{2p \times 2p}$ at the final iteration of RFM on modular addition/subtraction is of the form
\begin{align}
\label{eq: circulant ansatz}
M^* = \begin{pmatrix}
A  & C\tran\\
C & A
\end{pmatrix},
\end{align}
where $A, C \in \mathbb{R}^{p \times p}$ and $C$ is an asymmetric circulant matrix (i.e. $C$ is not a degenerate circulant such as $C=I$ or $C=\one\one\tran$). Furthermore, we note that  $A$ is of the form $c_1 I + c_2 \one\one\tran$ for constants $c_1, c_2$.
\label{circulant_ansatz} 
\end{observation} 

Similarly to addition and subtraction,  RFM successfully learns multiplication and division. Yet, in contrast to addition and subtraction,
the structure of feature matrices for these tasks, shown in  Fig.~\ref{fig:p61_kernel_agops}B, is not at all obvious. 
Nevertheless, re-ordering the rows and columns of the feature matrices for these tasks brings out their hidden circulant structure  of the form stated in Eq.~\eqref{eq: circulant ansatz}.
We show the effect of re-ordering in  Fig.~\ref{fig:p61_kernel_agops}C (see also Appendix Fig.~\ref{fig:reordered-vs-not-reordered} for the evolution of re-ordered and original features during training).
  We briefly discuss the reordering procedure below and provide further details in Appendix~\ref{app: reordering by generator}. 

Our reordering procedure uses the fact of group theory that the multiplicative group $\mathbb{Z}_p^*$ is a cyclic group of order $p-1$ (e.g.,~\cite{koblitz1994course}). By definition of the cyclic group, there exists at least one element $g \in \mathbb{Z}_p^*$, known as a \textit{generator}, such that $\mathbb{Z}_p^* = \{g^i ~;~ i \in \{1, \ldots, p-1\} \}$.  Using zero-indexing for our coordinates, we reorder rows and columns of the off-diagonal blocks of the feature matrices by moving the matrix entry in position $(r, c)$ with $r, c \in \mathbb{Z}_p^*$ to position $(i, j)$ where $r = g^i$ and $c = g^j$.  The entries in row zero and column zero are not reordered as they are identically zero (multiplying any integer by zero results in zero).  In the setting of modular multiplication/division, the map taking $g^i$ to $i$ is known as the \textit{discrete logarithm} base $g$~\cite[Ch.3]{koblitz1994course}. 

It is natural to expect block-circulant feature matrices to arise in modular multiplication/division after reordering by the discrete logarithm as (1) the discrete logarithm converts modular multiplication/division into modular addition/subtraction and (2) we  already observed block-circulant feature matrices in addition/subtraction in Fig.~\ref{fig:p61_kernel_agops}A.  
We  note the recent  work~\cite{doshi2024grokkingmodularpolynomials} also used the discrete logarithm to reorder coordinates in the context of constructing a solution for solving modular multiplication with neural networks.

\paragraph{Progress measures.} 
We will propose and examine two measures of feature learning, {\it circulant deviation} and {\it AGOP alignment}.

\emph{Circulant deviation.}
The fact that the final feature matrices contain circulant sub-blocks suggests a natural progress measure for learning modular arithmetic with RFM: namely, how far AGOP feature matrices are from a block-circulant matrix. 

 For a feature matrix $M$, let $A$ denote the bottom-left sub-block of $M$. We define circulant deviation as the total variance of the (wrapped) diagonals of $A$ normalized by the norm $\|A\|_F^2$. In particular, let $\mathcal{S} \in \Real^{p \times p} \rightarrow \Real^{p \times p}$ denote the shift operator, which shifts the $\ell$-th row of the matrix by  $\ell$ positions  to the right. Also let 
 $$\mathrm{Var}(\v) = \sum_{j=0}^{p-1} (v_j - \mathbb{E}\v)^2
 $$ be the variance of a vector $\v$. If $A[j]$ denotes the $j$-th column of $A$, we define circulant deviation $\mathcal{D}$ as 
\begin{align*}
    \mathcal{D}(A) = \frac{1}{\|A\|_F^2} \sum_{j=0}^{p-1} \mathrm{Var}(\mathcal{S}(A)[j]). \tag{Circulant deviation}
\end{align*}
Circulant deviation is a natural measure of how far a matrix is from a circulant,  since circulant matrices are constant along their (wrapped) diagonals and, therefore, have a circulant deviation of $0$. 

We see in Fig.~\ref{fig:p61_kernel_curves}B (top row) that circulant deviation exhibits gradual improvement through the course of training  with RFM. We find that for the first 10 iterations, while the training loss is numerically zero and the test loss does not improve, circulant deviation exhibits gradual, nearly linear, improvement. The improvements in circulant deviation reflect visual improvements in features, as was also shown in Fig.~\ref{fig:intro_fig_p59}. These curves also provide further support for Observation~\ref{circulant_ansatz}, as the circulant deviation is close to $0$ at the end of training. 

Note that circulant deviation is specific to modular arithmetic and depends crucially on the observation that the feature matrices contained circulant blocks (upon reordering for multiplication/division).  For more general tasks, we may not be able to identify such structure, and thus, we would need a more general progress measure that does not require a precise description of structures in generalizing features. To this end, we propose a second progress measure, AGOP alignment, that applies beyond modular arithmetic.

\emph{AGOP alignment.} Given two matrices $A, B \in \mathbb{R}^{d \times d}$, let $\rho(A, B)$ denote the standard cosine similarity between these two matrices when vectorized.  Specifically, let $\tilde{A}, \tilde{B} \in \mathbb{R}^{d^2}$ denote the vectorization of $A$ and $B$ respectively, then
\begin{align}
\label{eq: AGOP alignment}
    \rho(A,B) = \frac{\langle \tilde{A},\tilde{B} \rangle }{\|\tilde{A}\|\,\|\tilde{B}\|}~.
\end{align}
If $M_t$ denotes the AGOP at iteration $t$ of RFM (or epoch $t$ of a neural network) and $M^*$ denotes the final AGOP of the trained RFM (or neural network), then AGOP alignment at iteration $t$ is given by $\rho(M_t,M^*)$. The same measure of alignment was  used in \cite{catapultsAGOP}, except that alignment in~\cite{catapultsAGOP} was computed with respect to the AGOP of the ground truth model. Note that as modular operations are discrete, in our setting there is no unique ground truth model for which AGOP can be computed.  

Like circulant deviation, AGOP alignment exhibits gradual improvement in the regime that test loss is constant and large (see Fig.~\ref{fig:p61_kernel_curves}B, bottom row). Moreover, AGOP alignment is a more general progress measure since it does not require assumptions  on the structure of the AGOP. For instance, AGOP alignment can be measured without reordering for modular multiplication/division.  While AGOP alignment does not require a specific form of the final features, it is still an {\it a posteriori} measurement of progress as it requires access to the features of a fully trained model.

\begin{wrapfigure}{R}{0.45\textwidth}
    \centering
    \includegraphics[width=1.0\linewidth]{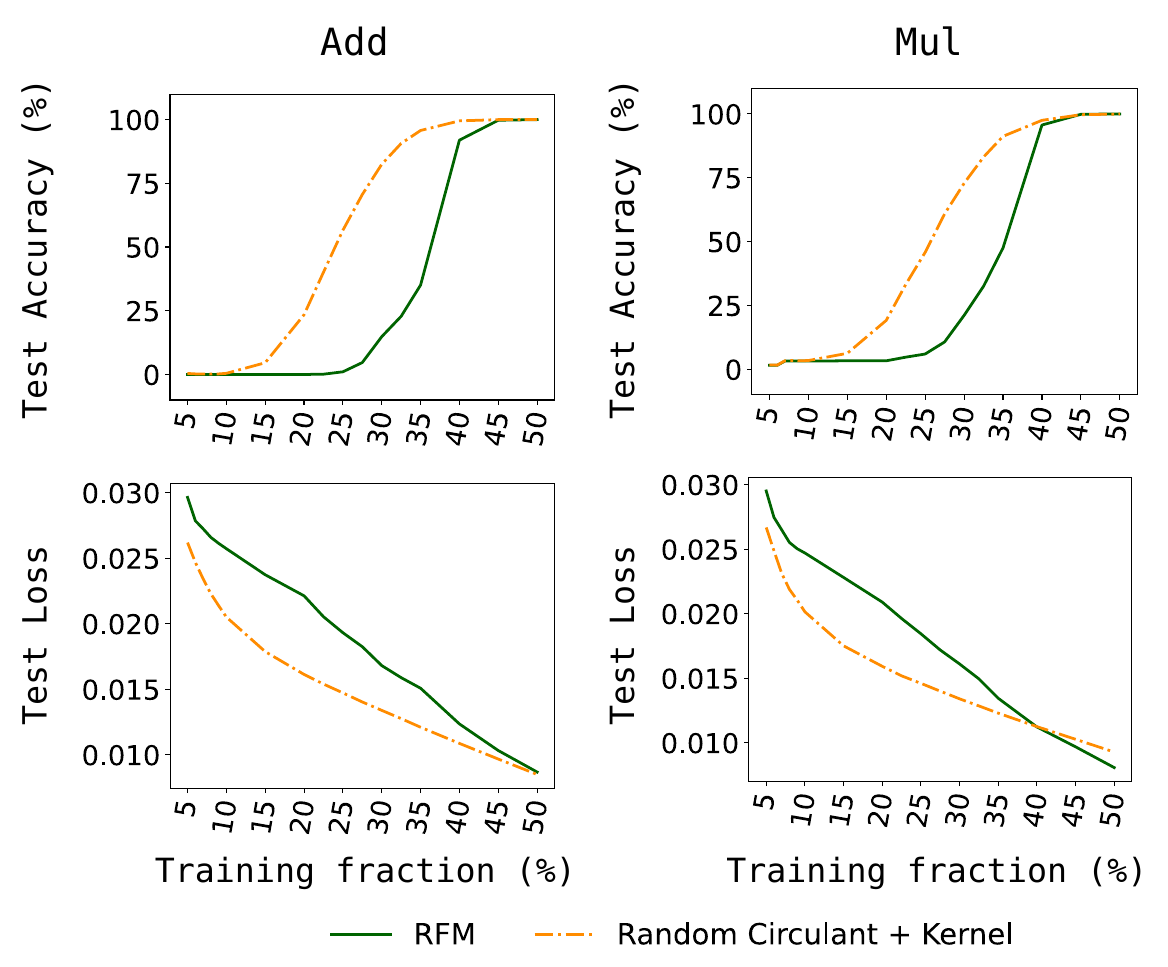}
    \caption{Random circulant features generalize with standard kernels for modular arithmetic tasks. RFM with the Gaussian kernel on modular addition (Add) and multiplication (Mul) for modulus $p = 61$ is compared to a standard Gaussian kernel machine trained on random circulant features (for Mul, the sub-blocks are circulant after re-ordering by the discrete logarithm base $2$).}
\label{fig: random_circulants}
\end{wrapfigure}

\textbf{Random circulant features allow standard kernels to generalize.}
We conclude this section by providing further evidence that the form of feature matrices given in Observation~\ref{circulant_ansatz} is key to enabling generalization in kernel machines trained to solve modular arithmetic tasks.  Note that in Observation~\ref{circulant_ansatz}, we only stated that the off-diagonal matrix $C$ was a circulant.  As we now show, a transformation with a {\it generic} block-circulant matrix  enables kernels machines to learn modular arithmetic. 
Namely, we generate a random circulant matrix $C$ by first sampling entries of the first column i.i.d. from the uniform distribution on $[0,1]\subset \R$ and then shifting the column to generate the remaining columns of $C$.
We proceed to construct the matrix $M^*$ in Observation~\ref{circulant_ansatz} with $c_1 = 1, c_2 = -\frac{1}{p}$.  For modular addition, we transform the input data by mapping $x_{ab} = \e_a \oplus \e_b$ to
\begin{align}
    \label{eq: random circ transform}
    \tilde{x}_{ab} = (M^*)^{\frac{1}{4}} x_{ab}~,
\end{align}
and then train on the new data pairs $(\tilde{x}_{ab}, \e_{a+b \mod p})$ for a subset of all possible pairs $(a, b) \in \Z_p^2$.  Note that transforming data with $(M^*)^{\frac{1}{4}}$ is akin to using $s = \frac{1}{2}$ in the RFM algorithm.  We do the same for modular multiplication after reordering the random circulant by the discrete logarithm as described above. 
The experiments in Fig.~\ref{fig: random_circulants} demonstrate that standard kernel machines trained on feature matrices with random circulant blocks\footnote{We have found that random Hankel features also enable generalization with standard kernel machines, similarly to random circulant features.}   outperform kernel-RFMs that  learn such features through AGOP. 

Additionally, we also find that directly enforcing circulant structure in the sub-matrices of $M_t$ throughout RFM iterations accelerates grokking and improves test loss (see Appendix~\ref{app: enforcing circulant}, Appendix Fig.~\ref{fig:verify_obs1}). 

These experiments provide strong evidence that the structure in Observation~\ref{circulant_ansatz} is key for generalization on modular arithmetic and, furthermore,  {\it no additional structure} beyond a generic circulant is required.   

\paragraph{Remark: multiple skills.} Throughout this section, we focused on modular arithmetic settings for a single task.  In more general domains such as language, one may expect there to be many ``skills'' that need to be learned.  In such settings, it is possible that these skills are grokked at different rates.  While a full discussion is beyond the scope of this work, to illustrate this behavior, we performed additional experiments in Appendix~\ref{app: multitask grokking} and Appendix Fig.~\ref{fig:multitask} where we train RFM on a pair of modular arithmetic tasks simultaneously and demonstrate that different tasks are indeed grokked at different points throughout training.

\section{Emergence in neural networks through AGOP}
\label{sec: neural nets}

\begin{figure}[t]
    \centering
    \includegraphics[width=1.0\textwidth]{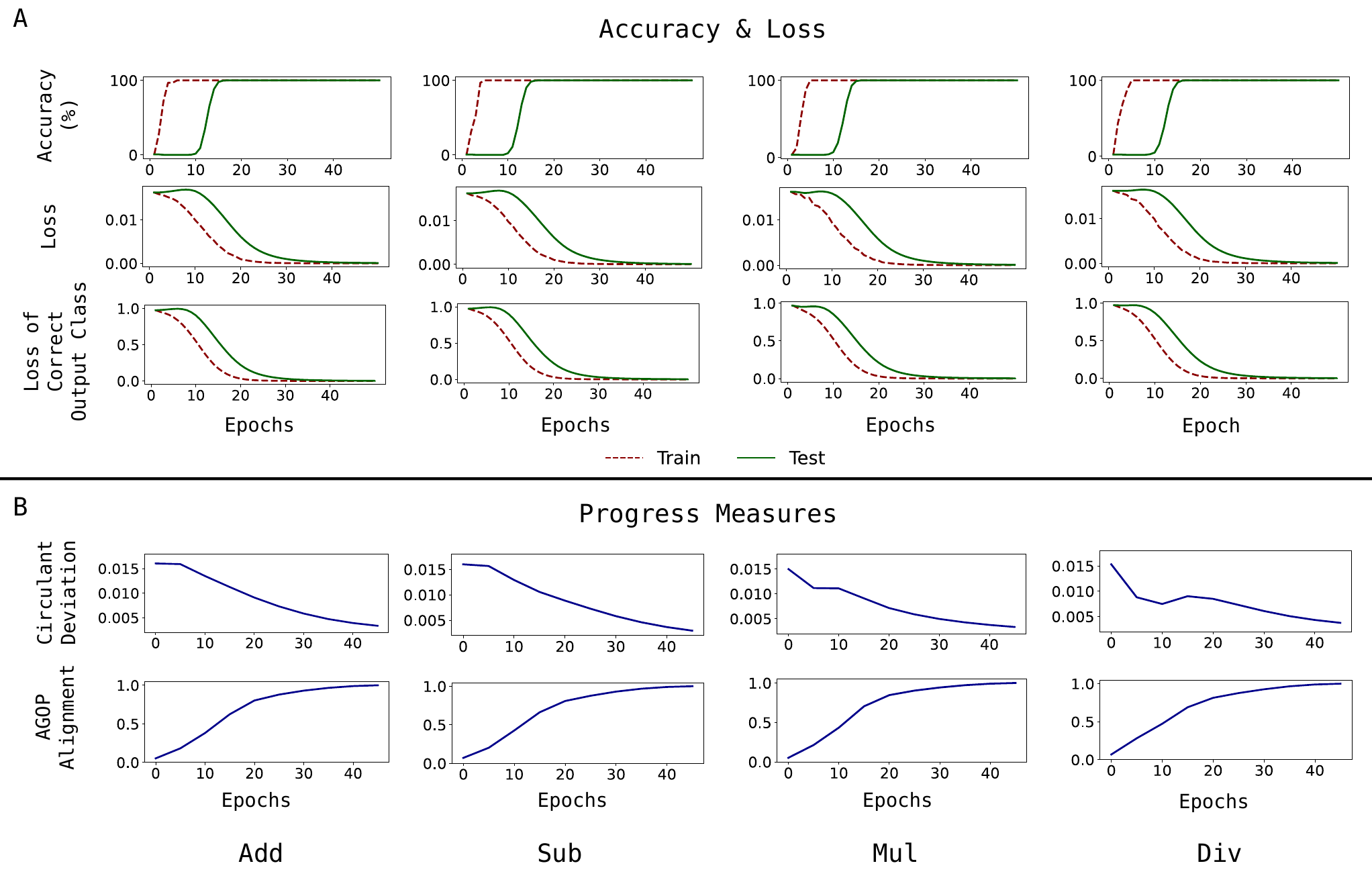}
    \caption{One hidden layer fully-connected networks with quadratic activations trained on modular arithmetic with modulus $p = 61$ trained for 50 epochs with the square loss. 
    (A) Test accuracy, test loss over all outputs, and test loss of the correct class output do not change in the initial iterations.
    (B) Progress measures for circulant deviation and AGOP alignment. Circulant deviation for Mul and Div are computed after reordering by the discrete logarithm base $2$.
    }
    \label{fig:nn_progress_measures}
\end{figure}

We now show that grokking in two-layer neural networks relies on the same principles as grokking by RFM.  Specifically we demonstrate that (1)  block-circulant features are key to neural networks grokking modular arithmetic; and (2) our measures (circulant deviation and AGOP alignment) indicate gradual progress towards generalization, while standard measures of generalization exhibit sharp transitions. All experimental details are provided in Appendix~\ref{apdx:model_training_details}.

\paragraph{Grokking with neural networks.} We first reproduce grokking with modular arithmetic using fully-connected networks as identified in prior works (Fig.~\ref{fig:nn_progress_measures}A) \citep{gromov2023grokking}. In particular, we train one hidden layer fully connected networks $f: \mathbb{R}^{2p} \to \mathbb{R}^{p}$ of the form 
\begin{equation}
\label{eq: neural net}
    f(x) = W_2 \sigma (W_1  x)
\end{equation}
with quadratic activation $\sigma(z) = z^2$ on modulus $p=61$ data with a training fraction $50\%$.  We train networks using AdamW \citep{AdamW} with a batch size of 32. 

Consistent with prior work \citep{gromov2023grokking} and analogously to RFMs, neural networks exhibit an initial training period where the train accuracy reaches 100\%, while test accuracy is at 0\% and test loss does not improve (see Fig.~\ref{fig:nn_progress_measures}A).\footnote{Of course, unlike RFM, the training loss decreases gradually throughout the optimization process.}
After this point, we see that the accuracy rapidly improves to achieve perfect generalization. As we did for RFM, we verify that the sharp transition in test loss is not an artifact of averaging the loss over all output coordinates.  In the third row of Fig.~\ref{fig:nn_progress_measures}A we show that the test loss of the individual correct output coordinate closely tracks the total loss, exhibiting the same transition.

\begin{figure}[t!]
    \centering
    \includegraphics[width=0.85\textwidth]{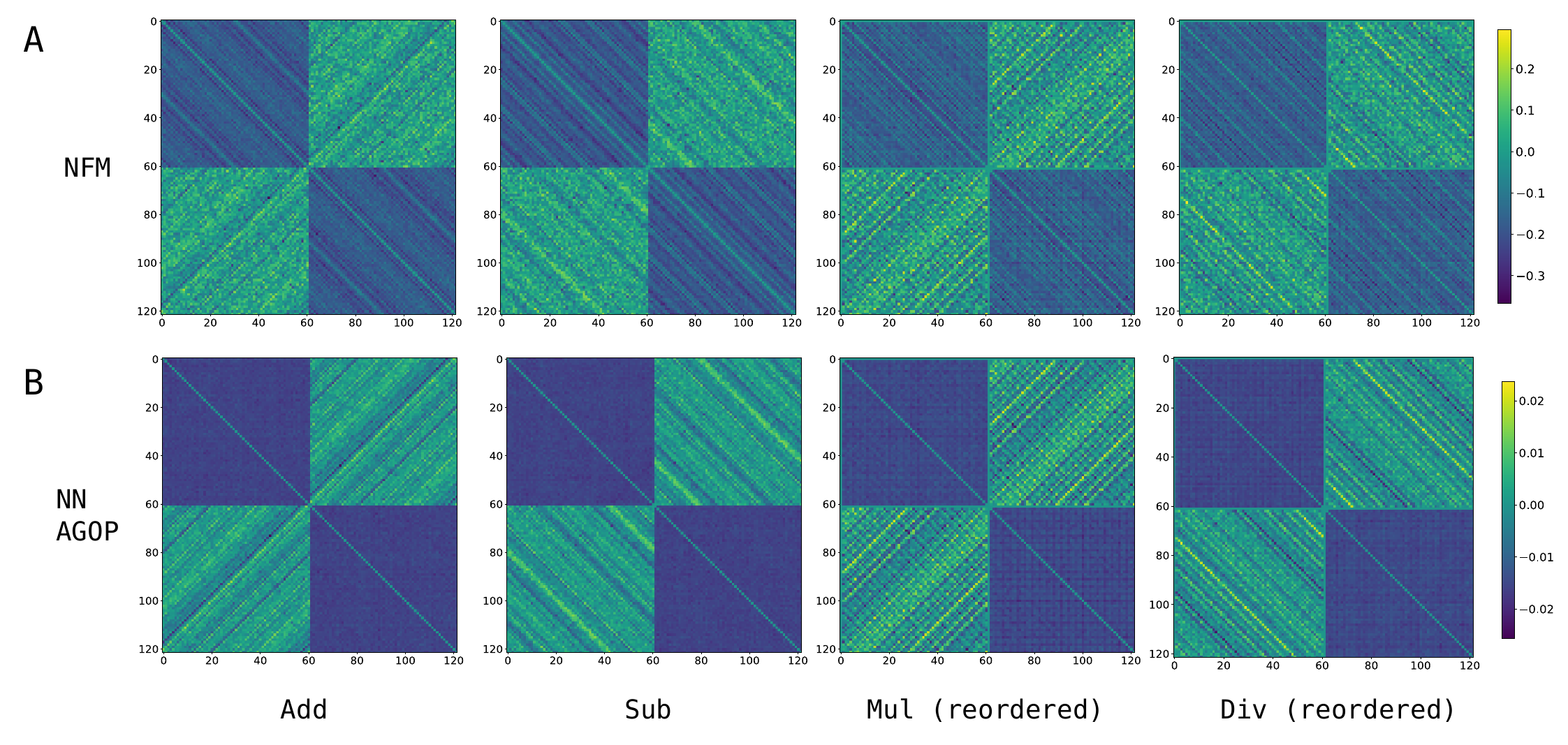}
    \caption{Feature matrices from one hidden layer neural networks with quadratic activations trained on addition, subtraction, multiplication, and division modulo $61$. 
    The Pearson correlations between the NFM and square root of the AGOP for each task are $0.955$ (Add), $0.942$ (Sub), $0.924$ (Mul), $0.929$ (Div).
    Mul and Div are shown after reordering by the discrete logarithm base $2$.
    }
    \label{fig:p61_nn_agop_nfm}
\end{figure}

\paragraph{Emergence of block-circulant features in neural networks.} In order to understand the features learned by neural networks for modular arithmetic, we visualize the first layer Neural Feature Matrix, which is defined as follows.  

\begin{definition}
\label{def: NFM}
Given a fully connected network of the form in Eq.~\eqref{eq: neural net}, the first layer \textit{Neural Feature Matrix} (NFM) is the matrix $W_1\tran W_1 \in \R^{2p \times 2p}$.
\end{definition}

We note the NFM is also the un-centered covariance of network weights and has been used in prior work in order to understand the features learned by various neural network architectures at any layer \citep{rfm_science,trockman2022understandingcovariancestructureconvolutional}. Fig.~\ref{fig:p61_nn_agop_nfm}A displays the NFM for one hidden layer neural networks with quadratic activations trained on modular arithmetic tasks.  For addition/subtraction, we find that the NFM exhibits block circulant structure, akin to the feature matrix for RFM.  As described in Section~\ref{sec: rfm emergence} and Appendix~\ref{app: reordering by generator}, we reorder the NFM for networks trained on multiplication/division with respect to a generator for $\Zp^*$ in order to observe block-circulant structure (see Appendix Fig.~\ref{fig:p61_nn_kernel_mul_div_reorder}A for a comparison of multiplication/division NFMs before and after reordering).  The block-circulant structure in both the NFM and the feature matrix of RFM  suggests that the two models are learning similar sets of features. 

Note that neural networks automatically learn features, as is demonstrated by visualizing the NFMs.  The work~\cite{rfm_science} posited that AGOP is the mechanism through which these models learn features.  Indeed, the authors stated their claim in the form of the Neural Feature Ansatz (NFA), which states that NFMs are proportional to a matrix power of AGOP through training (see Eq.~\eqref{eq: nfa} for a restatement of the NFA).  As such, we alternatively compute the square root of the AGOP to examine the features learned by neural networks trained on modular arithmetic tasks.  We visualize the square root of the AGOPs of these trained models in Fig.~\ref{fig:p61_nn_agop_nfm}B.  Corroborating the findings from~\cite{rfm_science}, we find that the square root of the AGOP and the NFM are highly correlated (greater than $0.92$), where Pearson correlation is equal to cosine similarity after centering the inputs to be mean $0$. Moreover, we find that the square root of AGOP of neural networks again exhibits the same structure as stated in Eq.~\eqref{eq: circulant ansatz} of Observation~\ref{circulant_ansatz} (see Appendix Fig.~\ref{fig:p61_nn_kernel_mul_div_reorder}B for a comparison of multiplication/division AGOPs before and after reordering). 

\begin{figure}
    \centering
    \includegraphics[width=0.9\linewidth]{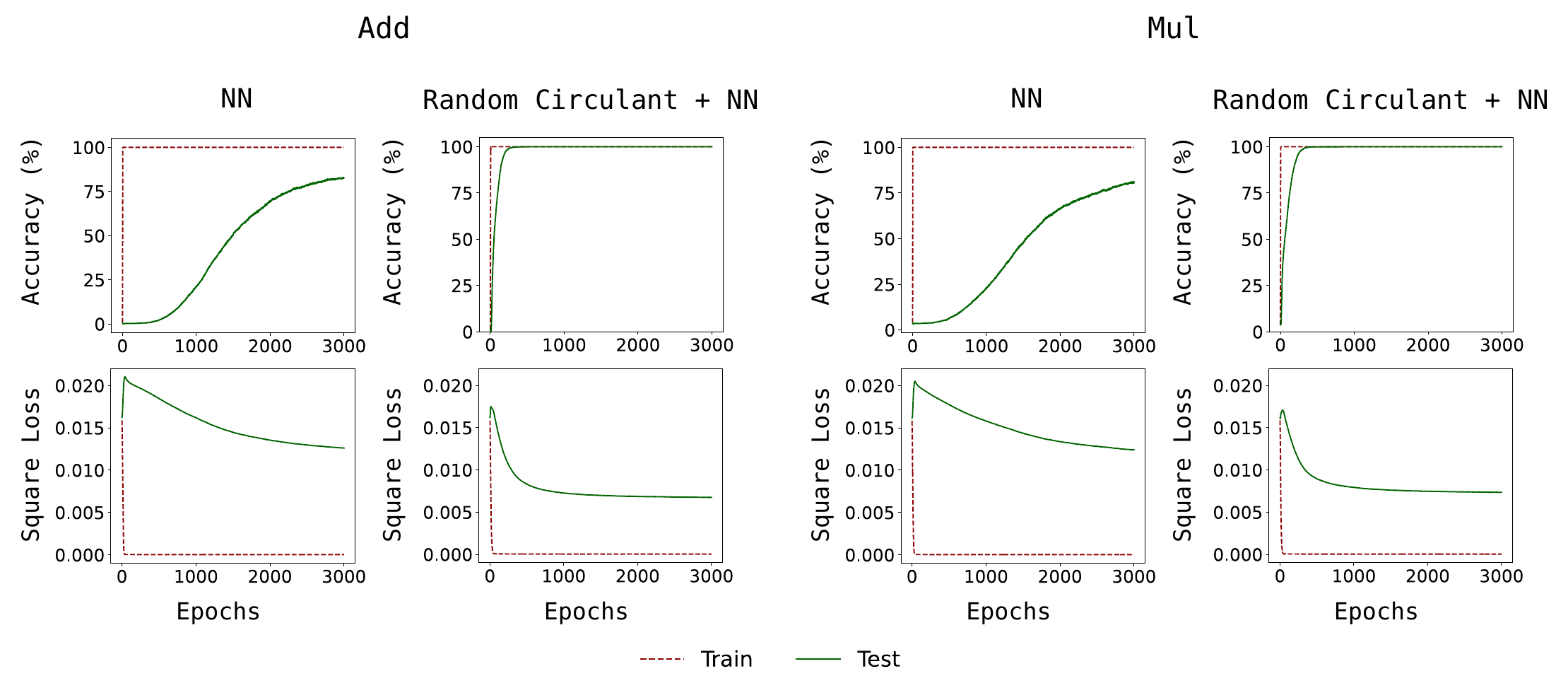}
    \caption{Random circulant features speed up generalization in neural networks for modular arithmetic tasks.
    We compare one hidden layer fully-connected networks with quadratic activations trained on modular addition and multiplication for $p = 61$ using standard one-hot encodings or using one-hot encodings transformed by random circulant matrices (re-ordered by the discrete logarithm in the case of multiplication).}
    \label{fig:nn_rand_circ}
\end{figure}

\paragraph{Random circulant maps improve generalization of neural networks.}
To further establish the importance of block-circulant features for solving modular arithmetic tasks with neural networks, we demonstrate that training networks on inputs transformed with a random block-circulant matrix greatly accelerates learning.  In  Fig.~\ref{fig:nn_rand_circ}, we compare the performance of neural networks trained on one-hot encoded modulo $p$ integers and the same integers transformed using a random circulant matrix generated using the procedure in Eq.~\eqref{eq: random circ transform}.  At a training fraction of $17.5\%$, we find that networks trained on transformed integers achieved $100\%$ test accuracy within several hundred epochs and exhibit little delayed generalization while networks trained on non-transformed integers do not achieve $100\%$ test accuracy even within $3000$ epochs. 

\paragraph{Progress measures.} Given that the square root of the AGOP of neural networks exhibits block-circulant structure, we can use circulant deviation and AGOP alignment to measure gradual progress of neural networks toward a generalizing solution.  As before, we measure circulant deviation in the case of multiplication/division after reordering the feature submatrix by a generator of $\Zp^*$.  In Fig.~\ref{fig:nn_progress_measures}B, we observe that our measures indicate gradual progress in contrast to sharp transitions in the standard measures of progress shown in Fig.~\ref{fig:nn_progress_measures}A.  Indeed, there is a period of $5$-$10$ epochs where circulant deviation and AGOP alignment improve while test loss is large (and test accuracy is small) and does not improve. Therefore, as was the case of RFM, these metrics reveal gradual progress of neural networks toward  generalizing solutions. 

\paragraph{AGOP regularization and weight decay.}  It has been argued in prior work that weight decay ($\ell_2$ regularization on network weights) is necessary for grokking to occur when training neural networks for modular arithmetic tasks \citep{varma2023explaininggrokkingcircuitefficiency,davies2023unifyinggrokkingdoubledescent,nanda2023progress}. Under the NFA (Eq.~\eqref{eq: nfa}), which states that $W_1\tran W_1$ is proportional to a matrix power of $G(f)$, we expect that performing weight decay on the first layer, i.e., penalizing the loss by $\|W_1\|_F^2 = \tr(W_1\tran W_1)$, should behave similarly to penalizing the trace of the AGOP, $\tr(G(f))$, during training.\footnote{We note this regularizer been used prior work where AGOP is called the Gram matrix of the input-output Jacobian~\cite{hoffman2019robustlearningjacobianregularization}.}  To this end, we compare the impact of using (1) no regularization; (2) weight decay; and (3) AGOP regularization when training neural networks on modular arithmetic tasks.  In Appendix Fig.~\ref{fig:agop_decay}, we find that, akin to weight decay, AGOP regularization leads to grokking in cases where using no regularization results in no grokking and poor generalization. These results provide further evidence that neural networks solve modular arithmetic by using the AGOP to learn features.

\section{Fourier multiplication algorithm from circulant features}
\label{sec: fourier multiplication}

We have seen so far that features containing circulant sub-blocks enable generalization for RFMs and neural networks across modular arithmetic tasks. We now provide theoretical support that shows how kernel machines equipped with such circulant features learn generalizing solutions. In particular, we show that there exist block-circulant feature matrices, as in Observation~\ref{circulant_ansatz}, such that kernel machines equipped with these features and trained on all available data for a given modulus $p$ solve modular arithmetic through the \emph{Fourier Multiplication Algorithm} (FMA). Notably, the FMA has been argued both empirically and theoretically in prior works to be the solution found by neural networks to solve modular arithmetic \citep{nanda2023progress, zhong2024clock}.  For completeness, we state the FMA for modular addition/subtraction from~\cite{nanda2023progress} below. While these prior works write this algorithm in terms of cosines and sines, our presentation simplifies the statement by using the DFT.

\paragraph{Fourier Multiplication Algorithm for modular addition/subtraction.}
Consider the modular addition task with $f^*(a,b) = (a+b) \mod p$. For a given input $x = x_{[1]} \oplus x_{[2]} \in \Real^{2p}$, the FMA generates a value for output class $\ell$, $y_{\mathrm{add}}(x;\ell)$, through the following computation: 
\begin{enumerate}
    \item Compute the Discrete Fourier Transform (DFT) for each digit vector $x_{[1]}$ and $x_{[2]}$, which we denote $\hat{x}_{[1]} = F x_{[1]}$ and $\hat{x}_{[2]} = F x_{[2]}$ where the matrix $F$ is defined in Eq.~\eqref{eq: DFT}. 
    \item Compute the element-wise product $\hat{x}_{[1]} \odot \hat{x}_{[2]}$.
    \item Return $\sqrt{p} ~\cdot~  \langle \hat{x}_{[1]} \odot \hat{x}_{[2]}, F \e_\ell \rangle_{\mathbb{C}}$ where $\e_{\ell}$ denotes $\ell$-th standard basis vector and $\langle \cdot , \cdot \rangle_{\mathbb{C}}$ denotes the complex inner product (see Eq.~\eqref{eq: Complex inner product}). 
\end{enumerate}
This algorithmic process can be written concisely in the following equation:
\begin{align}
\label{eqn: fma, addition}
    y_{\mathrm{add}}(x;\ell) = \sqrt{p} \cdot \inner{Fx_{[1]}\odot Fx_{[2]}, F\e_{\ell}}_\C~.
\end{align}
Note that for $x = \e_a \oplus \e_b$, the second step of the FMA reduces to 
\begin{align}
F\e_a \odot F\e_b = \frac{1}{\sqrt{p}} F\e_{(a+b) \mod p}~.
\end{align}
Using the fact that $F$ is a unitary matrix, the output of the FMA is given by 
\begin{align}
\sqrt{p} \cdot \inner{\frac{1}{\sqrt{p}}F\e_{(a+b) \mod p}, F\e_{\ell}}_\C = \e_{(a+b)\mod p}\tran F\tran \bar{F} \e_{\ell} = \e_{(a+b)\mod p}\tran \e_{\ell} =  \mathbbm{1}_{\{(a+b) \mod p = \ell\}}~.
\end{align}
Thus, the output of the FMA is a vector $\e_{(a+b) \mod p}$, which is equivalent to modular addition.   We provide an example of this algorithm for $p = 3$ in Appendix~\ref{app: FMA example}. 

\paragraph{Remarks.} We note that our description of the FMA uses all entries of the DFT, referred to as frequencies, while the algorithm as proposed in prior works allows for utilizing a subset of frequencies. Also note that the FMA for subtraction, written $y_{\mathrm{sub}}$, is similar and given by 
\begin{align}
\label{eqn: fma, subtraction}
y_{\mathrm{sub}}(x;\ell) = \sqrt{p} \cdot \inner{Fx_{[1]}\odot F\e_{p-\ell-1}, Fx_{[2]}}_\C~.
\end{align}

Having described the FMA, we now state our theorem.
\begin{theorem}
\label{thm: kernel fma}
Given all of the discrete data $\curly{\round{\e_a \oplus \e_b, \e_{(a-b) \mod p}}}_{a,b=0}^{p-1}$, for each output class $\ell \in \{0,\cdots,p-1\}$, suppose we train a separate kernel predictor 
$ f_\ell(x) = k(x,X;M_\ell)\alpha^{(\ell)}$ where $k(\cdot;\cdot;M_\ell)$ is a quadratic kernel with $
M_\ell = \begin{pmatrix}
0 & C^{\ell}\\
(C^\ell)\tran & 0
\end{pmatrix}$~
and $C \in \Real^{p \times p}$ is a circulant matrix with first row $\e_1$. When $\alpha^{(\ell)}$ is the solution to kernel ridgeless regression for each $\ell$, the kernel predictor $f = [f_0, \ldots, f_{p-1}]$ is equivalent to Fourier Multiplication Algorithm for modular subtraction (Eq.~\eqref{eqn: fma, subtraction}).
\end{theorem}

As $C$ is circulant, $C^\ell$ is also circulant. Hence, each $M_\ell$ has the structure described in Observation~\ref{circulant_ansatz}, where $A=0$. Note our construction differs from RFM in that we use a different feature matrix $M_\ell$ for each output coordinate, rather than a single feature matrix across all output coordinates.  Nevertheless, Theorem~\ref{thm: kernel fma} provides support for the fact that block-circulant feature matrices can be used to solve modular arithmetic.  

We provide the proof for Theorem~\ref{thm: kernel fma} for in Appendix~\ref{sec: proofs}. The argument for the FMA for addition (Eq.~\eqref{eqn: fma, addition}) is identical provided we replace $C^{\ell}$ with $C^{\ell} R$ and $(C^{\ell})\tran$ with $(C^{\ell}R)\tran$ in each $M_\ell$, where $R$ is the Hankel matrix that reverses the row order (i.e. ones along the main anti-diagonal, zero's elsewhere), whose first row is $\e_{p-1}$. An analogous result follows for multiplication and division under re-ordering by a group element, as described in Section~\ref{sec: rfm emergence}.

Our proof uses the well-known fact that circulant matrices can be diagonalized using the DFT matrix \citep{circulant_book} (see Lemma~\ref{fact:circulant, DFT} for a restatement of this fact). This fundamental relation intuitively connects circulant features and the FMA. By using kernels with block-circulant Mahalanobis matrices, we effectively represent the one-hot encoded data in terms of their Fourier transforms. We conjecture that this implicit representation is what enables RFM to learn modular arithmetic with more general circulant matrices when training on just a fraction of the discrete data. 

Not only do neural networks and RFM learn similar features, we now have established a setting where kernel methods equipped with block-circulant feature matrices learn the same out-of-domain solution as neural networks on modular arithmetic tasks. This result is interesting, in part, as the only constraint for generalization on these tasks is to obtain perfect accuracy on inputs that are standard basis vectors. However, as such functions can be extended arbitrarily over all of $\Real^{2d}$, there are infinitely many generalizing solutions. Therefore, the particular out-of-domain solution found by training is determined by the specifics of the learning algorithm. It is intriguing that kernel-RFMs and neural networks, which are clearly quite different algorithms, are both implicitly biased toward solutions that involve block-circulant feature matrices.

\section{Discussion and Conclusions}
\label{sec: discussion}

In recent years our understanding of generalization in machine learning has undergone dramatic changes.  Most classical analyses of generalization relied on the training loss serving as a proxy for the test loss and thus a useful measure of generalization.  Empirical results of deep learning have upended this {long-standing} belief. In many settings, predictors that interpolate the data (fit the data exactly) can still generalize, thus invalidating training loss as a possible predictor  of test performance. This has led to the recent developments in understanding benign overfitting, not just in neural networks but even in classical kernel and linear models~\cite{belkin2021fit,bartlett2021deep}. 
Since the training loss may not  be a  reliable predictor for generalization, the common  
suggestion has been to use the validation loss computed  on a separate {\it validation dataset}, one that is not used in training and is ideally indistinguishable from the test set. 
This procedure is standard practice: neural network training is typically stopped once the validation set loss stops improving. 

Emergent phenomena, such as grokking, show that we cannot rely even on validation performance at intermediate training steps to predict generalization at the end of training. Indeed, validation loss at a certain iteration may not be indicative of the validation loss itself only a few iterations later. Furthermore, it is clear that, contrary to~\cite{schaeffer2023are}, these phase transitions in performance are not generally ``a mirage'' since, as we demonstrate in this work, they are not predicted by standard {\it a priori} measures of performance, continuous or discontinuous. Instead, at least in our setting, emergence is fully determined by feature learning, which is difficult to observe without having access to a fully trained generalizable model. 
Indeed, the progress measures discussed in this work, as well as those suggested before in, e.g.,~\cite{barak2022hidden, nanda2023progress,doshi2024grokkingmodularpolynomials} can be termed  {\it a posteriori} progress indicators. They all require  either non-trivial understanding of the algorithm implemented by a fully generalizable trained model (such as our circulant deviation, the Fourier gap considered in~\cite{barak2022hidden}, or the Inverse Participation Ratio in~\cite{doshi2024grokkingmodularpolynomials}) or access to such a model (such as AGOP alignment).  
\begin{wrapfigure}{r}{0.3\textwidth}
  \centering    \includegraphics[width=1\linewidth]{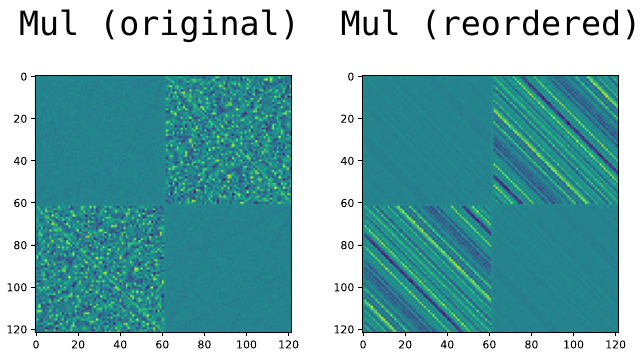}
    \caption{Square root of AGOP for quadratic RFM trained on modular multiplication with $p=61$ before reordering (left) and after reordering (right).}
    \label{fig:conclusion-reordered-vs-not-reordered}
\end{wrapfigure}
To sharpen this last point, consider generalizable features for modular multiplication shown in Fig.~\ref{fig:conclusion-reordered-vs-not-reordered} on the right. Aside from the block structure, the original features shown in the left panel of Fig.~\ref{fig:conclusion-reordered-vs-not-reordered} do not have a visually identifiable pattern.  In contrast, re-ordered features in the right panel are clearly striped and are thus immediately suggestive of block-circulants. Note that, as discussed in Section~\ref{sec: rfm emergence}, reordering of features
requires understanding that the multiplicative group $\Z_p^*$ is cyclic of order $p-1$. While it is a well-known result in group theory, it is far from obvious {\it a priori}. It is thus plausible that in other settings hidden feature structures may be hard to identify due to a lack of comparable mathematical insights. 

\paragraph{Why is learning modular arithmetic surprising?}
The task of learning modular operations  appears to be fundamentally different from many other statistical machine learning tasks.   
In continuous ML settings, we typically posit that the ``ground truth'' target function is smooth in an appropriate sense. 
Hence any general purpose algorithm capable of learning smooth functions (such as, for  example, $k$-nearest neighbors) should be able to learn the target function given enough data. 
Primary differences between learning algorithms are  thus in sample and computational efficiency. 
In contrast,  it is not clear what principle leads to learning modular arithmetic from partial observations. There are many ways to fill in the missing data and we do not know a  simple inductive bias, such as smoothness, to guide us toward a solution. Several recent works argued that margin maximization with respect to certain norms can account for learning modular arithmetic~\citep{morwani2024featureemergencemarginmaximization, lyu2023dichotomy, mohamadi2024why}. While the direction is promising,  general underlying principles have not yet been elucidated. 

\paragraph{Low rank learning.} 
The problem of  learning modular arithmetic can be viewed as a type of matrix completion -- completing the   $p\times p$ matrix (so-called Cayley table) representing modular operations, from partial observations. 
The best studied  matrix completion problem is low rank matrix completion, where the goal is to  fill in missing entries of a low rank matrix from observing a subset of the entries~\cite[Ch.8]{moitra2018algorithmic}. 
 While many specialized algorithms exist, it has been observed that neural networks can  recover low rank matrix structures~\cite{gunasekar2017implicit}. Notably, in a development paralleling the results of this paper, low-rank matrix completion can provably be performed by linear RFMs using the same AGOP mechanism~\cite{LinearRFM}.  

It is thus tempting to posit that grokking modular operations in neural networks or RFM can be  explained as a low rank prediction problem.  Indeed  modular operations can be implemented by an index $4$ model, i.e., a function of the form $f=g(Ax)$, where $x\in \R^{2p}$ and $A$ is a rank $4$ matrix (see Appendix~\ref{app:low_rank} for the construction). It is a plausible conjecture as there is strong evidence, empirical and theoretical, that neural networks are capable of learning such multi-index models~\cite{damianLowRank, IoannisLowRank} as well as low-rank matrix completion. Furthermore, a  phenomenon similar to grokking was discussed in~\cite[Fig. 5, 6]{radhakrishnan2022mechanism}  in the context of low rank feature learning for both neural networks and RFM. However, despite the existence of generalizeable low rank models, the actual circulant features learned by both Neural Networks and RFM are {\it not} low rank.  

Interestingly, this observation mirrors the problem of  learning parity functions through neural network inspired minimum norm interpolation, which was analyzed in~\cite{pmlr-v195-ardeshir23a}. While single-directional (index one) solutions exist in that setting, the authors show that the minimum norm solutions are all multi-dimensional.  

\paragraph{Emergence with respect to compute, training data size and model size.} In this paper we have primarily dealt with emergence 
at some point in the course of the training process. In that setting (e.g., Fig.~\ref{fig:intro_fig_p59}) we can consider the $x$ axis as measurement of  the amount of compute, analogous to the number of epochs in training neural networks.
The dependence of model quality on compute is a question of key practical importance as every training process has to eventually stop. What if just one more iteration endowed the model with new amazing skills?  
As we have shown, there is no simply identifiable progress measure without access to the trained model or some insight into the algorithm it implements. 
The evidence is more mixed for emergence  with respect to the training data. As we show in Appendix Fig.~\ref{fig:emerge_train_data_frac}, while the test accuracy improves fairly sharply at about $25\%$ training fraction threshold, the  test loss exhibits gradual improvement as a function of the training data with no obvious phase transitions. Thus, unlike the emergence with respect to compute,  emergence with respect to the training data size may be a ``mirage'' in the sense of~\cite{schaeffer2023are}.  As far as the model size is concerned, we note that  kernel machines can be thought of as neural networks of infinite width. We are thus not able to analyze emergence with respect to the model size in these experiments. 

\paragraph{Analyses of grokking.}
Recent works~\citep{kumar2024grokking, lyu2023dichotomy, mohamadi2024why} argue that grokking occurs in neural networks through a two phase mechanism that transitions from a ``lazy'' regime, with no feature learning, to a ``rich''  feature learning regime. Our experiments clearly show that grokking in RFM does not undergo such a transition. For RFM on modular arithmetic tasks, our progress measures indicate that the features evolve  gradually  toward the final circulant matrices,  even as test performance initially remains constant (Fig.~\ref{fig:p61_kernel_curves}). Grokking in these settings is entirely due to the gradual feature quality improvement and two-phase grokking does not occur. Therefore,  two-phase  cannot serve as  a general explanation of grokking with respect to the compute.
Additionally, we have not observed significant evidence of ``lazy'' to ``rich'' transition as a mechanism for grokking in our experiments with neural networks, as most of our measures of feature learning start improving early on in the training process (improvement in circulant deviation measure is delayed for addition and subtraction, but not for multiplication and division, while AGOP feature alignment is initially nearly linear for all tasks), see Fig.~\ref{fig:nn_progress_measures}. {These observations for neural networks are in line with the results in \citep{doshi2024grokkingmodularpolynomials, nanda2023progress}, where their proposed progress measures, Inverse Participation Ratio and Gini coefficients of the weights in the Fourier domain, are shown to increase prior to improvements in test loss and accuracy for modular multiplication and addition}.

Furthermore, as grokking modular arithmetic occurs in a kernel model equipped with a linear feature learning mechanism, a general explanation for grokking cannot depend on mechanisms that are specific to neural networks. Therefore, explanations for grokking that depend on the magnitude of the weights or neural circuit efficiency (e.g., \citep{nanda2023progress, varma2023explaininggrokkingcircuitefficiency}) or other attributes of neural networks, such as specific optimization methods, cannot account for the phenomena described in our work. 

\paragraph{Conclusions.} 
In this paper, we showed that grokking modular arithmetic happens in feature learning kernel machines in a manner very similar to what has been observed in neural networks. Perhaps the most unexpected aspect of our findings is that feature learning can happen independently of improvements in both training and test loss. {Note that this is hard to observe in the context of neural networks as non-zero training loss is required  for the training process.} Not only does this finding reinforce the narrative of rapid emergence of skills in neural networks, it is also  not easily explicable within the framework of the existing generalization theory. 

Finally, this work  adds to the growing body of evidence that the AGOP-based mechanisms of feature learning  can account for some of the most interesting  phenomena in deep learning. These include
generalization with multi-index models~\citep{ParkinsonEGOP}, deep neural collapse~\citep{agopNC}, and the ability to perform low-rank matrix completion~\citep{LinearRFM}. Thus, RFM provides a framework that is both practically powerful and serves as a theoretically tractable model of deep learning. 

\section*{Acknowledgements}

We acknowledge support from the National Science Foundation (NSF) and the Simons Foundation for the Collaboration on the Theoretical Foundations of Deep Learning through awards DMS-2031883 and \#814639 as well as the  TILOS institute (NSF CCF-2112665). This work used the programs (1) XSEDE (Extreme science and engineering discovery environment)  which is supported by NSF grant numbers ACI-1548562, and (2) ACCESS (Advanced cyberinfrastructure coordination ecosystem: services \& support) which is supported by NSF grants numbers \#2138259, \#2138286, \#2138307, \#2137603, and \#2138296. Specifically, we used the resources from SDSC Expanse GPU compute nodes, and NCSA Delta system, via allocations TG-CIS220009.
NM and AR gratefully acknowledge funding and support for this research from the Eric and Wendy
Schmidt Center at the Broad Institute of MIT and Harvard.
The authors would like to  thank Jonathan Xue for assistance in the initial phase of the project, Daniel Hsu for insightful comments and references, and Darshil Doshi and Tianyu He for useful discussion on related literature.
NM would additionally like to thank Sarah Heller for editing suggestions.

\bibliographystyle{abbrv}
\bibliography{aux/refs} 

\newpage
\appendix
\captionsetup[figure]{labelformat=default, labelsep=colon, name=Appendix Figure }
\setcounter{figure}{0}  

\section{Neural Feature Ansatz}
\label{app: nfa}

While the NFA has been observed generally across depths and architecture types \citep{rfm_science, convRFM, BeagleholeNFA}, we restate this observation for fully-connected networks with one hidden-layer of the form in Eq.~\eqref{eq: neural net}.

\begin{ansatz}[Neural Feature Ansatz for one hidden layer]
For a one hidden-layer neural network $\fnn$ and a matrix power $\alpha \in (0,1]$, the following holds:
\begin{align}
\label{eq: nfa}
    W_1\tran W_1 \propto \AGOPnn(\fnn)^s~.
\end{align}
\end{ansatz}
Note that this statement implies that $W_1\tran W_1$ and $\AGOPnn(\fnn)^s$ have a cosine similarity of $\pm1$.

In this work, we choose $\alpha = \frac{1}{2}$, following the main results in \cite{rfm_science}. While the absolute value of the cosine similarity is written in Eq.~\eqref{eq: nfa} to be $1$, it is typically a high value less than $1$, where the exact value depends on choices of initialization, architecture, dataset, and training procedure. For more understanding of these conditions, see~\cite{BeagleholeNFA}. \section{Model and training details}
\label{apdx:model_training_details}

\paragraph{Gaussian kernel:}
Throughout this work we take bandwidth $L = 2.5$ when using the Mahalanobis Gaussian kernel.
We solve ridgeless kernel regression using NumPy on a standard CPU.

\paragraph{Neural networks:}
Unless otherwise specified, we train one hidden layer neural networks with quadratic activation functions and no biases in PyTorch on a single A100 GPU.
Models are trained using AdamW with hidden width $1024$, batch size $32$, learning rate of $10^{-3}$, weight decay $1.0$, and standard PyTorch initialization.
All models are trained using the Mean Squared Error loss function (square loss).

For the experiments in Appendix Fig.~\ref{fig:agop_decay}, we train one hidden layer neural networks with quadratic activation and no biases on modular addition modulo $p = 61.$
We use 40\% training fraction, PyTorch standard initialization, hidden width of $512$, weight decay $10^{-5}$, and AGOP regularizer weight $10^{-3}$.
Models are trained with vanilla SGD, batch size $128$, and learning rate $1.0$.

 \section{Reordering feature matrices by group generators}
\label{app: reordering by generator}

Our reordering procedure uses the standard fact of group theory that the multiplicative group $\mathbb{Z}_p^*$ is a cyclic group of order $p-1$~\cite{koblitz1994course}. By definition of the cyclic group, there exists at least one element $g \in \mathbb{Z}_p^*$, known as a \textit{generator}, such that $\mathbb{Z}_p^* = \{g^i ~;~ i \in \{1, \ldots, p-1\} \}$.

Given a generator $g \in \mathbb{Z}_p^*$, we reorder features according to the map, $\phi_g: \mathbb{Z}_p^* \to \mathbb{Z}_p^*$, where if $h = g^{i}$, then $\phi_g(h) = i$.  In particular, given a matrix $B \in \mathbb{R}^{p \times p}$, we reorder the bottom right $(p - 1) \times (p -1) $ sub-block of $B$ as follows: we move the entry in coordinate $(r, c)$ with $r, c \in \mathbb{Z}_p^*$ to coordinate $(\phi_g(r), \phi_g(c))$. For example if $g = 2$ in $\mathbb{Z}_5^*$, then $(2, 3)$ entry of the sub-block would be moved to coordinate $(1, 3)$ since $2^1 = 2$ and $2^3 \mod 5 = 3$. In the setting of modular multiplication/division, the map $\phi_g$ defined above is known as the \textit{discrete logarithm} base $g$~\cite[Ch.3]{koblitz1994course}.  The discrete logarithm is analogous to the logarithm defined for positive real numbers in the sense that it converts modular multiplication/division into modular addition/subtraction.  Lastly, in this setting, we note that we only reorder the bottom $(p-1) \times (p-1)$ sub-block of $B$ as the first row and column are $0$ (as multiplication by $0$ results in $0$).

Upon re-ordering the $p \times p$ off-diagonal sub-blocks of the feature matrix by the map $\phi_g$, the feature matrix of RFM for multiplication/division tasks contains circulant blocks as shown in Fig.~\ref{fig:p61_kernel_agops}C.  Thus, the reordered feature matrices for these tasks also exhibit the structure in Observation~\ref{circulant_ansatz}. As a remark, we note that there can exist several generators for a cyclic group, and thus far, we have not specified the generator $g$ we use for re-ordering.  For example, $2$ and  $3$ are both generators of $\mathbb{Z}_5^*$ since $\{2, 2^2, (2^3 \mod 5), (2^4 \mod 5)\} = \{3, (3^2 \mod 5), (3^3 \mod 5), (3^4 \mod 5)\} = \mathbb{Z}_5^*$.   Lemma~\ref{lemma: circulant all or none} implies that the choice of generator does not matter for observing circulant structure.  As a convention, we simply reorder by the smallest generator.
 \section{Enforcing circulant structure in RFM}
\label{app: enforcing circulant}

We see that the structure in Observation~\ref{circulant_ansatz} gives generalizing features on modular arithmetic when the circulant $C$ is constructed from the RFM matrix. 
We observe that enforcing this structure at every iteration, and comparing to the standard RFM model at that iteration, improves test loss and accelerates grokking on e.g. addition (Appendix Fig.~\ref{fig:verify_obs1}). 
The exact procedure to enforce this structure is as follows. We first perform standard RFM to generate feature matrices $M_1, \ldots, M_T$. 
Then for each iteration of the standard RFM, we construct a new $\wt{M}_t$ on which we solve ridgeless kernel regression for a new $\alpha$ and evaluate on the test set. 
To construct $\wt{M}$, we take $D = \diag{M_t}$ and first let $\wt{M} = D^{-1/2} M D^{-1/2}$, to ensure the rows and columns have equal scale. 
We then reset the top left and bottom right sub-matrices of $\wt{M}$ as $I - \frac{1}{p} \mathbf{1}\mathbf{1}^T$, and replace the bottom-left and top-right blocks with $C$ and $C\tran$, where $C$ is an exactly circulant matrix constructed from $M_t$. 
Specifically, where $\c$ is the first column of the bottom-left sub-matrix of $M_t$, column $\ell$ of $C$ is equal to $\sigma^\ell(M_t)$.  \section{Grokking multiple tasks}
\label{app: multitask grokking}

We train RFM to simultaneously solve the following two modular polynomial tasks: (1) $x + y \mod p$ ; (2) $x^2 + y^2 \mod p$ for modulus $p = 61$.
We train RFM with the Mahalanobis Gaussian kernel using bandwidth parameter $L = 2.5$.
Training data for both tasks is constructed from the same 80\% training fraction.
In addition to concatenating the one-hot encodings for $x, y$, we also append an extra bit indicating which task to solve ($0$ indicating task (1) and $1$ indicating task (2)). 
The classification head is shared for both tasks (e.g. output dimension is still $\R^p$).

In Appendix Fig.~\ref{fig:multitask}, we observe that there are two sharp transitions in the test loss and test accuracy.  By decomposing the loss into the loss per task, we observe that RFM groks task (1) prior to grokking task (2).  Overall, these results illustrate that grokking of different tasks can occur at different training iterations.   

 \section{FMA example for $p = 3$}
\label{app: FMA example}

We now provide an example of the FMA for $p = 3$.  Let $x = \e_1 \oplus \e_2$.  In this case, we expect the FMA to output the vector $\e_{0}$ since $(1 + 2) \mod 3 = 0$.  Following the first step of the FMA, we compute  
\begin{align}
    \hat{x}_{[1]} = F\e_1 = \frac{1}{\sqrt{3}}[1,\omega,\omega^2]\tran ~~;~~ \hat{x}_{[2]} = F\e_2 = \frac{1}{\sqrt{3}} [1,\omega^2,\omega^4]\tran ~,
\end{align}
which are the first and second columns of $F$, respectively. Then their element-wise product is given by 
\begin{align}
    F\e_1 \odot F\e_2 = \frac{1}{3}[1, \omega^3, \omega^6]\tran = \frac{1}{3}[1,1,1]\tran = \frac{1}{\sqrt{3}} F\e_{0}~,
\end{align}
which is $\frac{1}{\sqrt{3}}$ times the first column of the DFT matrix.  Finally, we compute the outputs $\sqrt{3} \inner{\frac{1}{\sqrt{3}}F\e_{0}, F\e_{\ell}}_\C$ for each $\ell \in \{0,1,2\}$. As $F$ is unitary, $y_{\mathrm{add}}(\e_1 \oplus \e_2;\ell) = \mathbbm{1}_{\curly{1+2 = \ell \mod 3}}$, so that coordinate $0$ of the output will have value $1$, and all other coordinates have value $0$.  
 \section{Additional results and proofs}
\label{sec: proofs}

\begin{lemma}
\label{lemma: circulant all or none}
    Let $C \in \Real^{p \times p}$ with its first row and column entries all equal to $0$. Let the $(p-1)\times(p-1)$ sub-block starting at the second row and column be $C^\times$. Then, $C^\times$ is either circulant after re-ordering by any generator $q$ of $\Z_p^*$, or $C^\times$ is not circulant under re-ordering by any such generator.
\end{lemma}
\begin{proof}[Proof of Lemma~\ref{lemma: circulant all or none}]
We prove the lemma by showing that for any two generators $q_1, q_2$ of $\Z_p^*$, if $C^\times$ is circulant re-ordering with $q_1$, then it is also circulant when re-ordering by $q_2$.

Suppose $C^\times$ is circulant re-ordering with $q_1$. Let $i,j \in \{1,\ldots,p-1\}$. Note that by the circulant assumption, for all $s \in \Z$,
\begin{align}
    C_{q_1^i,q_1^j} = C_{q_1^{i+s},q_1^{i+s}}~,
\end{align}
where we take each index modulo $p$.

As $q_2$ is a generator for $\Zp^*$, we can access all entries of $C^\times$ by indexing with powers of $q_2$. Further, as $q_1$ is a generator, we can write $q_2 = q_1^k$, for some power $k$. Let $a \in \Z$. Then,
\begin{align*}
    C_{q_2^i, q_2^j} &= C_{q_1^{ki}, q_1^{kj}}\\
    &= C_{q_1^{ki + ka}, q_1^{kj + ka}}\\
    &= C_{q_1^{k(i + a)}, q_1^{k(j + a)}}\\
    &= C_{q_2^{i + a}, q_2^{j + a}}~.
\end{align*}
Therefore, $C$ is constant on the diagonals under re-ordering by $q_2$, concluding the proof.
\end{proof}

We next state Lemma~\ref{fact:circulant, DFT}, which is used in the proof of Theorem~\ref{thm: kernel fma}.

\begin{lemma}[See, e.g., \cite{circulant_book}]
\label{fact:circulant, DFT}
Circulant matrices $U$ can be written (diagonalized) as:
\begin{align*}
    U = F D \herm{F}~,
\end{align*}
where $F$ is the DFT matrix, $\herm{F}$ is the element-wise complex conjugate of $F\tran$ (i.e. the Hermitian of $F$), and $D$ is a diagonal matrix with diagonal $\sqrt{p} \cdot Fu$, where $u$ is the first row of $U$.
\end{lemma}

We now present the proof of Theorem~\ref{thm: kernel fma}, restating the theorem below for the reader's convenience. 

\begin{theorem*}
Given all of the discrete data $\curly{\round{\e_a \oplus \e_b, \e_{(a-b) \mod p}}}_{a,b=0}^{p-1}$, for each output class $\ell \in \{0,\cdots,p-1\}$, suppose we train a separate kernel predictor 
$ f_\ell(x) = k(x,X;M_\ell)\alpha^{(\ell)}$ where $k(\cdot;\cdot;M_\ell)$ is a quadratic kernel with $
M_\ell = \begin{pmatrix}
0 & C^{\ell}\\
(C^\ell)\tran & 0
\end{pmatrix}$~
and $C \in \Real^{p \times p}$ is a circulant matrix with first row $\e_1$. When $\alpha^{(\ell)}$ is the solution to kernel ridgeless regression for each $\ell$, the kernel predictor $f = [f_0, \ldots, f_{p-1}]$ is equivalent to Fourier Multiplication Algorithm for modular subtraction (Eq.~\eqref{eqn: fma, subtraction}).
\end{theorem*}

\begin{proof}[Proof of Theorem~\ref{thm: kernel fma}]
We present the proof for modular subtraction as the proof for addition follows analogously.  We write the standard kernel predictor for class $\ell$ on input $x = x_{[1]} \oplus x_{[2]} \in \Real^{2p}$ as,
\begin{align*}
    f_\ell(x) = \sum_{a,b=0}^{p-1} \alpha^{(\ell)}_{a,b} k\round{x, \e_a \oplus \e_b; M_\ell}~,
\end{align*}
where we have re-written the index into kernel coefficients for class $\ell$, $\alpha^{(\ell)} \in \Real^{p \times p}$, so that the coefficients are multi-indexed by the first and second digit. Specifically, now $\alpha^{(\ell)}_{a,b}$ is the kernel coefficient corresponding to the representer $k(\cdot, x)$ for input point $x = \e_a \oplus \e_b$. Recall we use a quadratic kernel, $k(x,z;M_\ell) = (x\tran M_\ell z)^2$. In this case, the kernel predictor simplifies to,
\begin{align*}
    f_\ell(x) = \sum_{a,b=0}^{p-1} \alpha^{(\ell)}_{a,b} \round{x_{[1]}\tran C^{\ell} \e_b + \e_a\tran C^{\ell} x_{[2]}}^2~.
\end{align*}
Then, the labels for each pair of input digits, written as a matrix $Y^{(\ell)} \in \Real^{p \times p}$ for the $\ell$-th class where the row and column index the first and second digit respectively, are $Y^{(\ell)} = C^{-\ell}$.  

For $x = \e_{a'} \oplus \e_{b'}$, i.e. $x$ in the discrete dataset, we have,
\begin{align*}
    f_\ell(x) &= \sum_{a,b=0}^{p-1} \alpha^{(\ell)}_{a,b} \round{\delta_{(a,b'-\ell)} + \delta_{(a',b-\ell)} + 2\delta_{(a,b'-\ell)}\delta_{(a',b-\ell)}}\\
    &= \e_{b'-\ell}\tran \alpha^{(\ell)} \one + \one\tran \alpha^{(\ell)} \e_{a' + \ell} + 2 \e_{b'-\ell}\tran \alpha^{(\ell)} \e_{a' + \ell}\\
    &= \e_{b'}\tran C^{-\ell} \alpha^{(\ell)} \one +  \one\tran \alpha^{(\ell)} C^{-\ell} \e_{a'} + 2 \e_{b'}\tran C^{-\ell} \alpha^{(\ell)} C^{-\ell} \e_{a'}\\
    &=\e_{b'}\tran\round{C^{-\ell} \alpha \one\one\tran + \one\one\tran \alpha C^{-\ell} + 2 C^{-\ell} \alpha C^{-\ell}}\e_{a'} ~,
\end{align*}
where $\delta_{(u,v)} = \mathbbm{1}_{\curly{u = v}}$. Let $f_\ell(X) \in \Real^{p \times p}$ be the matrix of function values of $f_\ell$, where $[f_\ell(X)]_{a,b} = f_\ell(\e_a \oplus \e_b)$, and, therefore, $f_\ell(\e_a \oplus \e_b) = \e_a\tran f_\ell(X) \e_b$. Then, to solve for $\alpha^{(\ell)}$, we need to solve the system of equations for $\alpha$,
\begin{align*}
    &f_\ell(X) = \round{C^{-\ell} \alpha \one\one\tran + \one\one\tran \alpha C^{-\ell} + 2 C^{-\ell} \alpha C^{-\ell}}\tran = C^{-\ell}\\
    &\iff C^{-\ell} \alpha \one\one\tran + \one\one\tran \alpha C^{-\ell} + 2 C^{-\ell} \alpha C^{-\ell} = C^{\ell}
\end{align*}
Note, by left-multiplying both sides by $C^{-\ell}$, we see this equation holds iff,
\begin{align*}
    C^{-2\ell} \alpha \one\one\tran + \one\one\tran \alpha C^{-\ell} + 2 C^{-2\ell} \alpha C^{-\ell} = I~.
\end{align*}
Note the solution is unique as the kernel matrix is full rank. We posit the solution $\alpha$ such that $C^{-2\ell} \alpha C^{-\ell} = \frac{1}{2}I + \lambda \one\one\tran$, which is $\alpha = \frac{1}{2}C^{3\ell} + \lambda \one\one\tran$. Then, solving for $\lambda$, we require,
\begin{align*}
    \one\one\tran + 2 p\lambda \one\one\tran + 2\lambda \one\one\tran = 0~,
\end{align*}
which implies $\lambda = -\frac{2}{2p + 2}$. Substituting this value of $\lambda$ and simplifying, we see finally that $f_\ell(x) = x_{[1]}\tran C^{-\ell} x_{[2]}$. Therefore, using that circulant matrices are diagonalized by $C = \sqrt{p} FD\herm{F}$ (Lemma~\ref{fact:circulant, DFT}) and $\herm{F} F = I$, where $D = \diag{F\e_1}$, we derive,
\begin{align*}
     f_\ell(x) &= \sqrt{p} \cdot x_{[1]}\tran F D^{-\ell} \herm{F} x_{[2]}\\
     &= \sqrt{p} \cdot x_{[1]}\tran F \diag{F\e_{p-\ell-1}} \herm{F} x_{[2]}\\
     &= \sqrt{p} \cdot \inner{Fx_{[1]} \odot F\e_{p-\ell-1}, F x_{[2]}}_{\C}
\end{align*}
which is the output of the FMA on modular subtraction.
\end{proof}

 \section{Low rank solution to modular arithmetic}
\label{app:low_rank}

\paragraph{Addition} We present a solution to the modular addition task whose AGOP is low rank, in contrast to the full rank AGOP recovered by RFM and neural networks. 

We define the ``encoding'' map $\Phi: \R^p \to \C$ as follows. For a vector $\a = [a_0, \ldots, a_{p-1}]$,
$$
\Phi(\a)= \sum_{k=0}^{p-1} a_k \exp\round{\frac{k2\pi i}{p}}~.
$$
Notice that $\Phi$ is a linear map such that $\Phi(\e_k)=\exp\round{\frac{k2\pi i}{p}}$. Notice also that $\Phi$ is partially invertible with the ``decoding'' map $\Psi:\C \to \R^p$.
$$
\Psi(z) = \widetilde{\mathrm{max}}\round{\inner{z, \exp\round{\frac{0\cdot2\pi i}{p}}},\ldots \inner{z,\exp\round{\frac{(p-1)\cdot2\pi i}{p}}}}~.
$$ 
Above $\widetilde{\mathrm{max}}$ is a function that makes all entries zero except for the largest one and the inner product is the usual inner product in $\C$ considered as $\R^2$. Thus 
\begin{equation}
\label{eq:extension_to_2p}
\Psi\round{\exp\round{\frac{k\cdot2\pi i}{p}}}  = \e_k~.
\end{equation}
$\Psi$ is a nonlinear map $\C \to \R^p$.
While it is discontinuous but can easily be  modified to make it differentiable. 

By slight abuse of notation, we will  define $\Phi:\R^{p} \times \R^p\to\C^2$ on pairs:
$$
\Phi(\e_j,\e_k) = (\Phi(\e_j), \Phi (\e_k))~.
$$
This is still a linear map but now to $\C^2$. 

Consider now a quadratic map $\M$ on 
$\C^2\to \C$ given by complex multiplication:
$$
\M(z_1, z_2) = z_1z_2~.
$$

It is clear that the composition $\Psi\M\Phi$ implements modular addition
$$
\Psi\M\Phi(\e_j,\e_k) = \e_{(j+k)\mod p}
$$
Furthermore, since $\Phi$ is a liner map to a four-dimensional space, the AGOP of the composition $\Psi\M\Phi$ is of rank $4$. 

\paragraph{Multiplication} The construction is for multiplication is very similar with  modifications which we sketch below. We first re-order the non-zero coordinates by the discrete logarithm with base equal to a generator of the multiplicative group $\e_g$ (see Appendix~\ref{app: reordering by generator}), while keeping the order of index $0$. Then, we modify $\Phi$ to remove index $a_0$ from the sum for inputs $\a$. 
Thus for multiplication,
$$
\Phi(\a)= \sum_{k=1}^{p-1} a_k \exp\round{\frac{k\cdot 2\pi i}{p-1}}~,
$$
Hence that $\Phi(\e_0) = 0$, $\Phi(\e_g) = \exp\left(\frac{2\pi i}{p-1}\right)$ and $\Phi(\e_{g^k}) = \exp\left(\frac{k\cdot 2\pi i}{p-1}\right)$. We  extend $\Phi$ to $\R^p\times\R^p$ as in Eq.~\ref{eq:extension_to_2p} above. Note that $\Phi$ and the re-ordering together are still a linear map of rank $4$. 

Then, the ``decoding'' map, $\Psi(z)$, will be modified to return $0$, when $z=0$, and otherwise, 
$$
\Psi(z) = g^{\widetilde{\mathrm{max}}\round{\inner{z, \exp\round{\frac{0\cdot2\pi i}{p-1}}},\ldots \inner{z,\exp\round{\frac{(p-2)\cdot2\pi i}{p-1}}}}}~.
$$ 
$\M$ is still defined as above.
It is easy to check that the composition of $\Psi\M\Phi$ with reordering implements modular multiplication modulo $p$ and furthermore, the AGOP will also be of rank $4$.
 \clearpage

\begin{figure}
  \centering
    \includegraphics[width=0.75\linewidth]{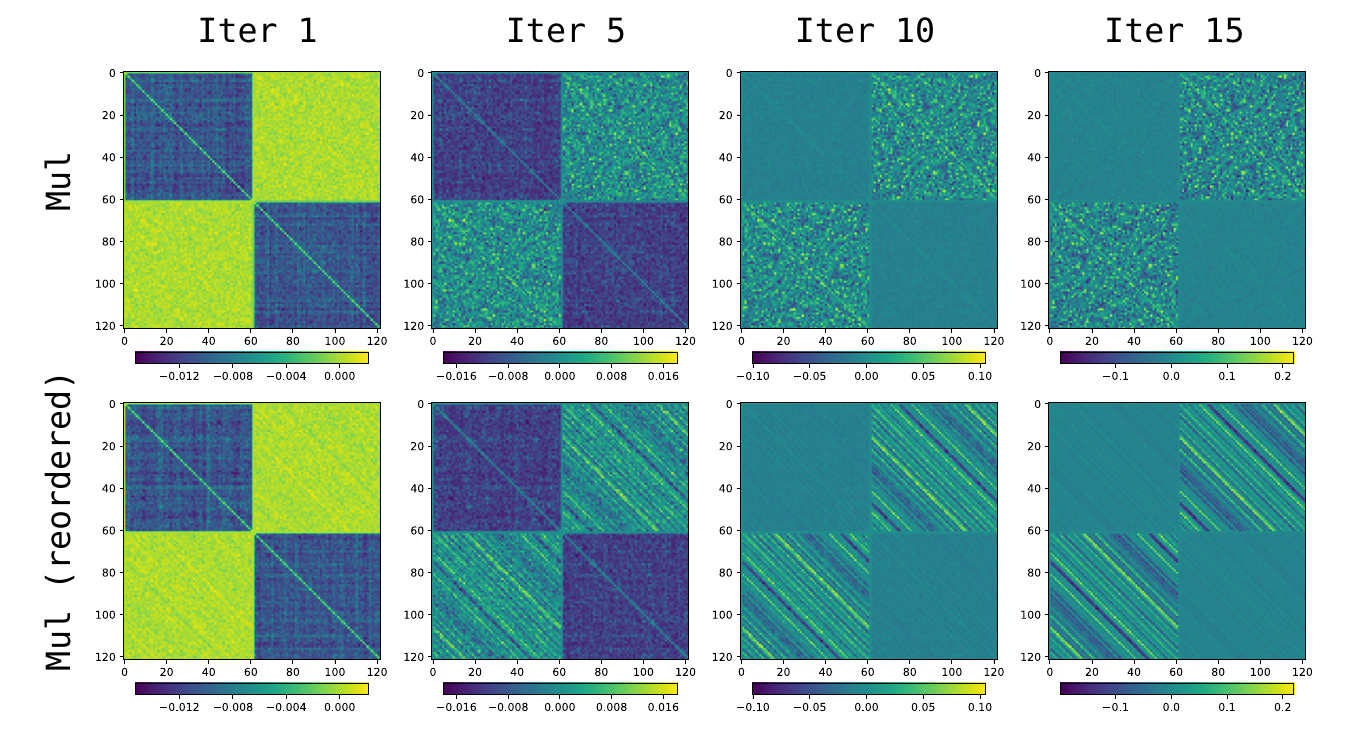}
    \caption{AGOP evolution for quadratic RFM trained on modular multiplication with $p=61$ before reordering (top row) and after reordering by the logarithm base $2$ (bottom row).}
    \label{fig:reordered-vs-not-reordered}
\end{figure}

\begin{figure}[b!]
    \centering
    \includegraphics[width=0.6\textwidth]{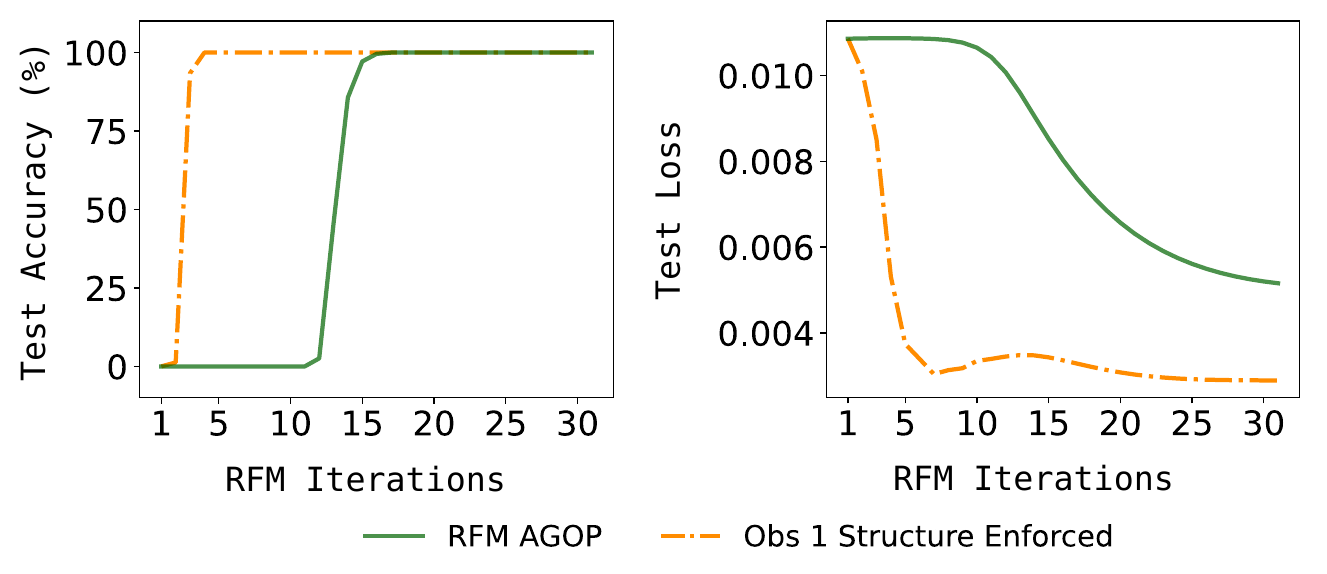}
    \caption{We train a Gaussian kernel-RFM on $x + y \mod 97$ and plot test loss and accuracy versus RFM iterations. We also evaluate the performance of the same model upon modifying the $M$ matrix to have exact block-circulant structure stated in Observation~\ref{circulant_ansatz}.}
    \label{fig:verify_obs1}
\end{figure}

\clearpage
\newpage

\begin{figure}[h!]
    \centering
    \includegraphics[width=1.0\textwidth]{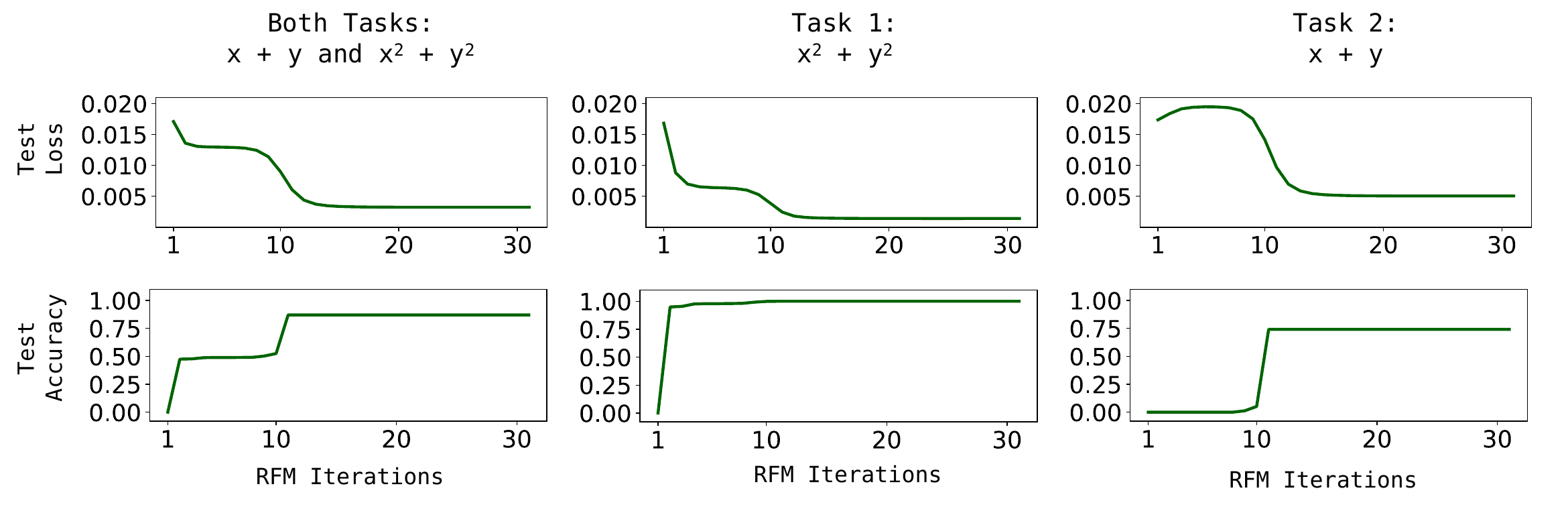}
    \caption{RFM with the Gaussian kernel trained on two modular arithmetic tasks with modulus $p = 61$. Task 1 is to learn $x^2 + y^2 \mod p$ and task 2 is to learn $x + y \mod p$.  
    }
    \label{fig:multitask}
\end{figure}

\begin{figure}[b!]
    \centering
    \includegraphics[width=0.85\textwidth]{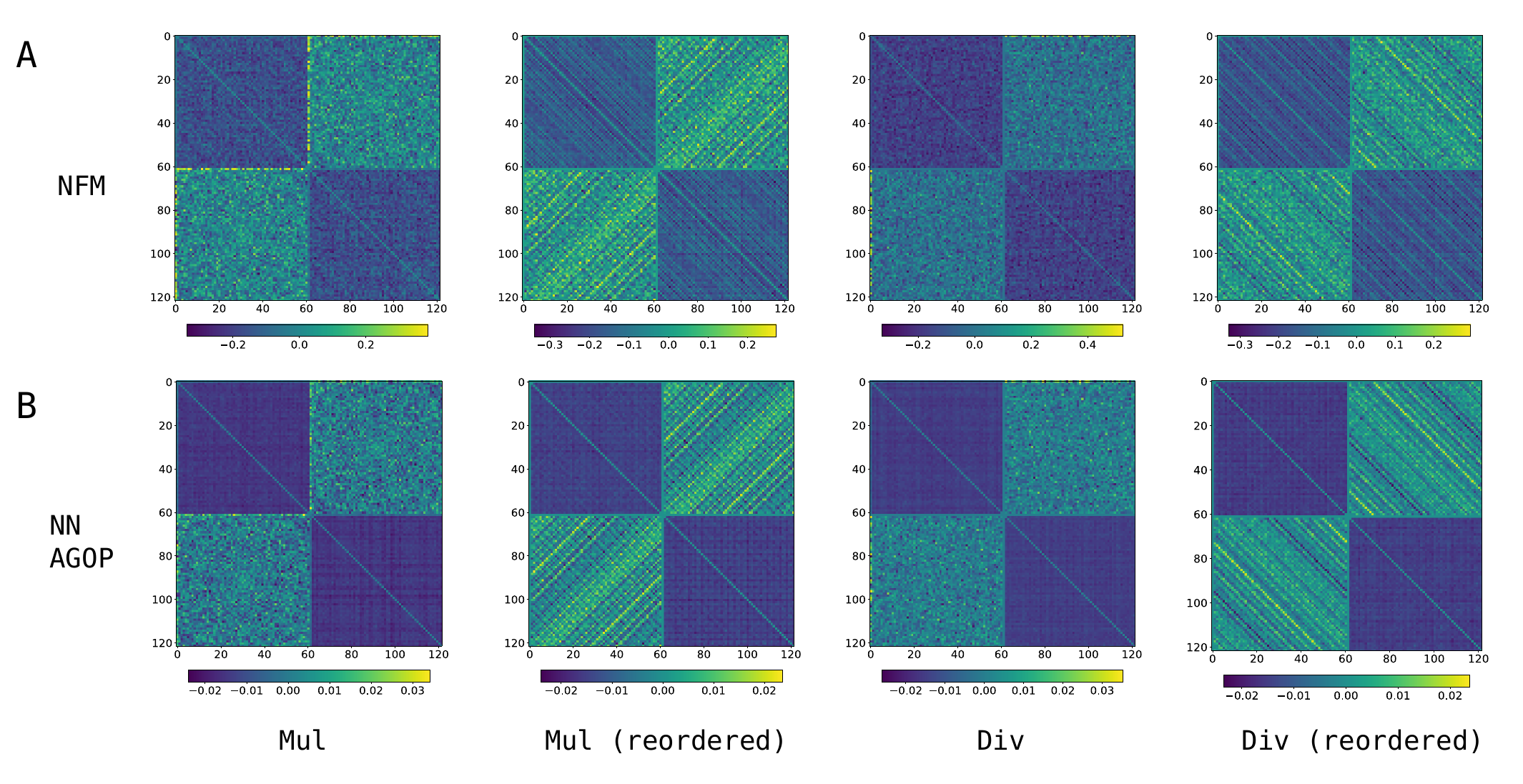}
    \caption{(A) We visualize the neural feature matrix (NFM) from a one hidden layer neural network with quadratic activations trained on modular multiplication and division, before and after reordering by the discrete logarithm.
    (B) We visualize the square root of the AGOP of the neural network in (A) before and after reordering.}
    \label{fig:p61_nn_kernel_mul_div_reorder}
\end{figure}

\newpage

\begin{figure}[h!]
    \centering
    \includegraphics[width=0.8\linewidth]{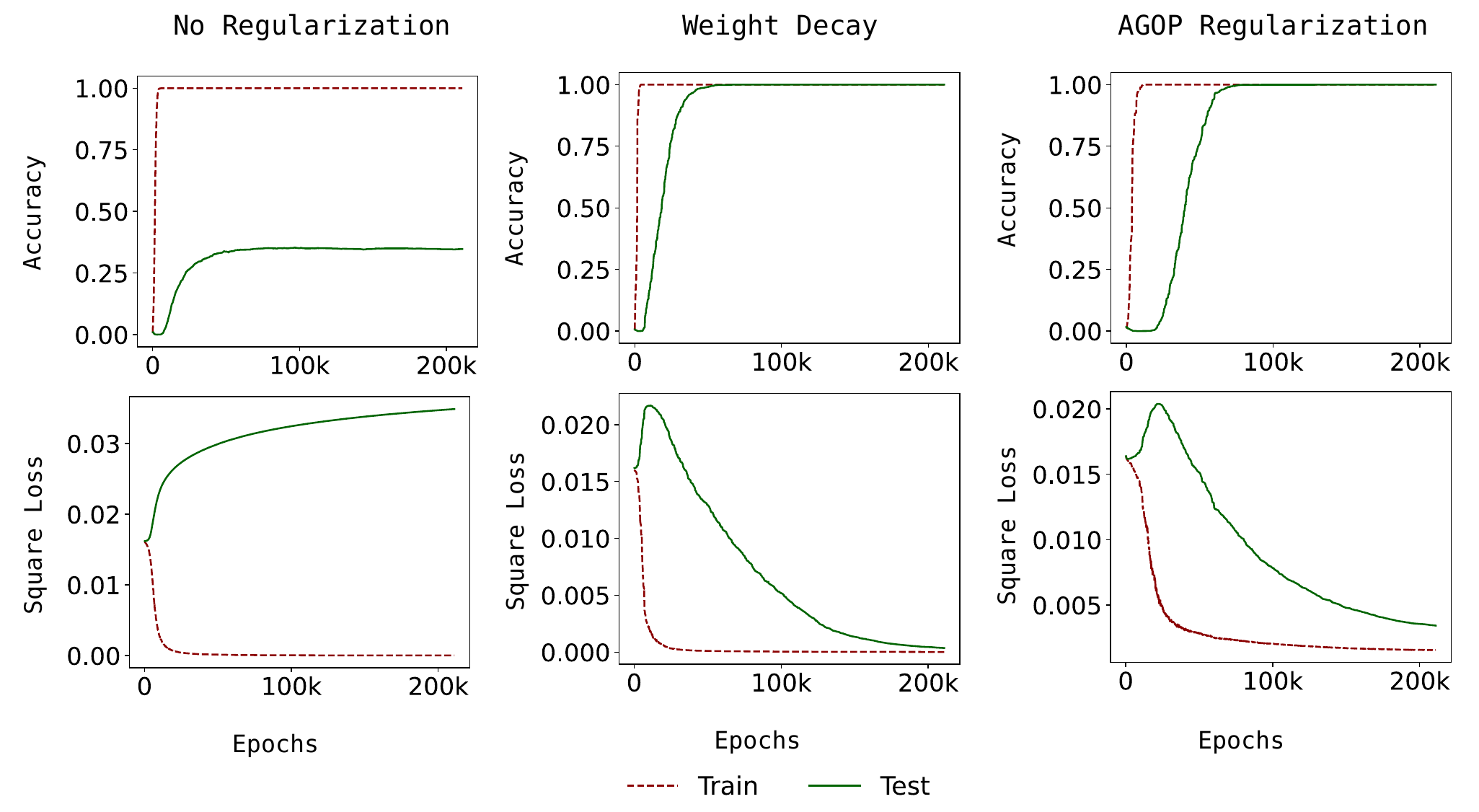}
    \caption{One hidden layer fully connected networks with quadratic activations trained on modular addition with $p = 61$ with vanilla SGD. Without any regularization the test accuracy does not go to $100\%$ whereas using weight decay or regularizing using the trace of the AGOP result in $100\%$ test accuracy and grokking.}
    \label{fig:agop_decay}
\end{figure}

\begin{figure}
  \centering
  \includegraphics[width=0.8\textwidth]{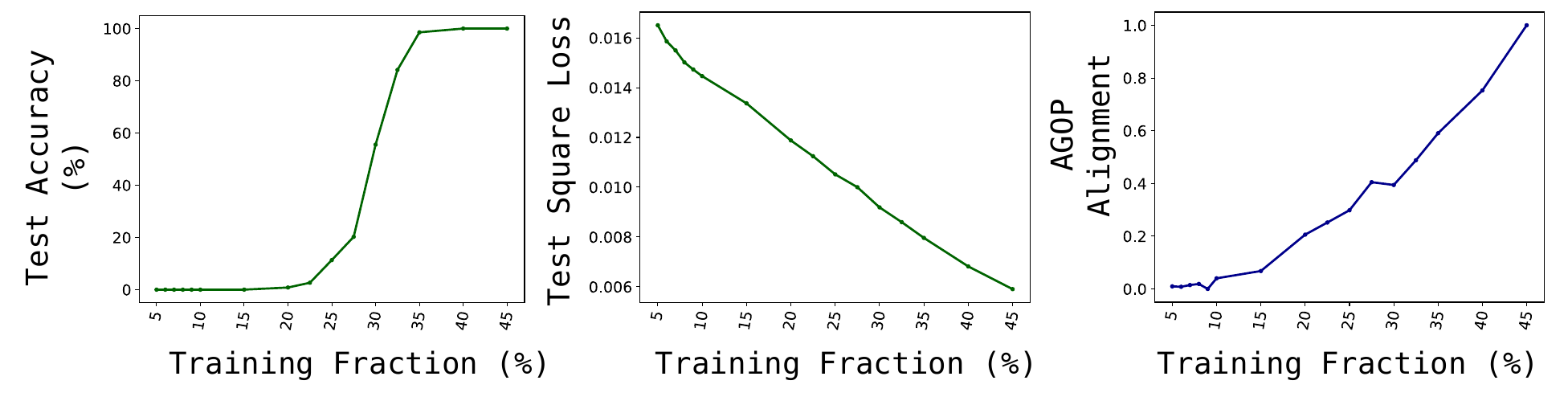}

  \caption{We train kernel-RFMs for 30 iterations using the Mahalanobis Gaussian kernel for $x + y \mod 97$.
  We plot test accuracy, test loss, and AGOP alignment versus percentage of training data used (denoted training fraction).
  All models reach convergence (i.e., both the test loss and test accuracy no longer change) after 30 iterations.
  We observe a sharp transition in test accuracy with respect to the training fraction, but we observe gradual change in test loss and AGOP alignment with respect to the training data fraction.}
  \label{fig:emerge_train_data_frac}
\end{figure}

 \end{document}